\documentclass[10pt]{article}

\usepackage{jmlr2e}
\usepackage[utf8]{inputenc} 
\usepackage[T1]{fontenc}    
\usepackage{hyperref}       
\usepackage{url}            
\usepackage{booktabs}       
\usepackage{amsfonts}       
\usepackage{nicefrac}       
\usepackage{microtype}      
\usepackage{amsmath,dsfont,empheq}
\usepackage{algorithm}
\usepackage{algpseudocode}
\usepackage{syntonly,bm,wrapfig,epsfig}
\usepackage{xcolor}
\input{mysymbol.sty}
\usepackage[bf]{caption} 

\usepackage{multibib}
\usepackage{graphicx}
\newcites{apndx}{References}

\def \bbeta {{\boldsymbol \beta}}

\algrenewcommand{\algorithmiccomment}[1]{\hskip0em$\triangleright$ #1}

\title{LAG: Lazily Aggregated Gradient for\\ Communication-Efficient Distributed Learning}

\author{
Tianyi Chen$^{\star}$ \qquad Georgios B. Giannakis$^{\star}$ \qquad  Tao Sun$^{\dag,\ast}$ \qquad Wotao Yin$^{\ast}$\\\\
{\centering
\addr $^{\star}$\textit{University of Minnesota - Twin Cities, Minneapolis, MN 55455, USA} \\
$^{\dag}$\textit{National University of Defense Technology, Changsha, Hunan 410073, China}\\
$^{\ast}$\textit{University of California - Los Angeles, Los Angeles, CA 90095, USA}\\
\email \{\texttt{chen3827,georgios\}@umn.edu}\qquad~~~\texttt{nudtsuntao@163.com}\qquad~~~ \texttt{wotaoyin@math.ucla.edu}\qquad~~}
}

\ShortHeadings{LAG: Lazily Aggregated Gradient for Communication-Efficient Distributed Learning}{Chen and Giannakis and Sun and Yin}
\firstpageno{1}

\begin{document}

%
%

\maketitle

\thispagestyle{empty}
\begin{abstract}
This paper presents a new class of gradient methods for distributed machine learning that adaptively skip the gradient calculations to learn with reduced communication and computation.
Simple rules are designed to detect slowly-varying gradients and, therefore, trigger the reuse of outdated gradients.
The resultant gradient-based algorithms are termed \textbf{L}azily \textbf{A}ggregated \textbf{G}radient --- justifying our acronym \textbf{LAG} used henceforth.
Theoretically, the merits of this contribution are: i) the convergence rate is the same as batch gradient descent in strongly-convex, convex, and nonconvex smooth cases;
and, ii) if the distributed datasets are heterogeneous (quantified by certain measurable constants), the communication rounds needed to achieve a targeted accuracy are reduced thanks to the adaptive reuse of \emph{lagged} gradients.
Numerical experiments on both synthetic and real data corroborate a significant communication reduction compared to alternatives.
\end{abstract}

\section{Introduction}

In this paper, we develop communication-efficient algorithms to solve the following problem
\begin{align}\label{opt0}
	\min_{\bbtheta\in \mathbb{R}^d}~{\cal L}(\bbtheta)~~~{\rm with}~~~{\cal L}(\bbtheta):=\sum_{m\in{\cal M}}{\cal L}_m(\bbtheta)
\end{align}
where $\bbtheta\in \mathbb{R}^d$ is the unknown vector, ${\cal L}$ and $\{{\cal L}_m, m\!\in\!{\cal M}\}$ are smooth (but not necessarily convex) functions with ${\cal M}:=\{1,\ldots,M\}$.
Problem \eqref{opt0} naturally arises in a number of areas, such as multi-agent optimization \citep{nedic2009}, distributed signal processing \citep{gg2016,schizas2008}, and distributed machine learning \citep{dean2012}.
Considering the distributed machine learning paradigm, each ${\cal L}_m$ is also a sum of functions, e.g., ${\cal L}_m(\bbtheta)\!:=\!\sum_{n\in{\cal N}_m}\!\!\ell_n(\bbtheta)$, where $\ell_n$ is the loss function (e.g., square or the logistic loss) with respect to the vector $\bbtheta$ (describing the model) evaluated at the training sample $\bbx_n$; that is, $\ell_n(\bbtheta):=\ell(\bbtheta;\bbx_n)$.
While machine learning tasks are traditionally carried out at a single server, for datasets with massive samples $\{\bbx_n\}$, running gradient-based iterative algorithms at a single server can be prohibitively slow; e.g., the server needs to sequentially compute gradient components given limited processors.
A simple yet popular solution in recent years is to parallelize the training across multiple computing units (a.k.a. workers) \citep{dean2012}.
Specifically, assuming batch samples distributedly stored in a total of $M$ workers with the worker $m\in{\cal M}$ associated with samples $\{\bbx_n,\,n\in {\cal N}_m\}$, a globally shared model $\bbtheta$ will be updated at the central server by aggregating gradients computed by workers.
Due to bandwidth and privacy concerns, each worker $m$ will not upload its data $\{\bbx_n, n\in{\cal N}_m\}$ to the server, thus the learning task needs to be performed by iteratively communicating with the server.

We are particularly interested in the scenarios where communication between the central server and the local workers is costly, as is the case with the Federated Learning paradigm \citep{mcmahan2017,smith2017}, and the cloud-edge AI systems \citep{stoica2017}.
In those cases, communication latency is the bottleneck of overall performance.
More precisely, the communication latency is a result of initiating communication links, queueing and propagating the message.
For sending small messages, e.g., the $d$-dimensional model $\bbtheta$ or aggregated gradient, this latency dominates the message size-dependent transmission latency.
Therefore, it is important to reduce the number of communication rounds, even more so than the bits per round.
In short, \textbf{our goal} is to find $\bbtheta$ that minimizes \eqref{opt0} using as low communication overhead as possible.

\subsection{Prior art}
To put our work in context, we review prior contributions that we group in two categories.

\paragraph{Large-scale machine learning.}
Solving \eqref{opt0} at a single server has been extensively studied for large-scale learning tasks, where the ``workhorse approach'' is the simple yet efficient stochastic gradient descent (SGD) \citep{robbins1951stochastic,bottou2010,bottou2016}.
For learning beyond a single server, distributed parallel machine learning is an attractive solution to tackle large-scale learning tasks, where the parameter server architecture is the most commonly used one \citep{dean2012,li2014}.
Different from the single server case, parallel implementation of the batch gradient descent (GD) is a popular choice, since SGD that has low complexity per iteration requires a large number of iterations thus communication rounds \citep{mcmahan2017blog}.
For traditional parallel learning algorithms however, latency, bandwidth limits, and unexpected drain on resources, that delay the update of even a single worker will slow down the entire system operation.
Recent research efforts in this line have been centered on understanding asynchronous-parallel algorithms to speed up machine learning by eliminating costly synchronization; e.g.,  \citep{cannelli2016,sun2017,peng2016,recht2011,liu2015jmlr}.

\paragraph{Communication-efficient learning.}
Going beyond single-server learning, the high communication overhead becomes the bottleneck of the overall system performance \citep{mcmahan2017blog}.
Communication-efficient learning algorithms have gained popularity \citep{jordan2018,zhang2013}.
Distributed learning approaches have been developed based on quantized (gradient) information, e.g., \citep{suresh2017}, but they only reduce the required bandwidth per communication, not the rounds.
For machine learning tasks where the loss function is convex and its conjugate dual is expressible, the dual coordinate ascent-based approaches have been demonstrated to yield impressive empirical performance \citep{smith2017,jaggi2014,ma2017}. But these algorithms run in a double-loop manner, and the communication reduction has not been formally quantified.
To reduce communication by accelerating convergence, approaches leveraging (inexact) second-order information have been studied in \citep{shamir2014,zhang2015icml}.
Roughly speaking, algorithms in \citep{smith2017,jaggi2014,ma2017,shamir2014,zhang2015icml} reduce communication by increasing local computation (relative to GD), while our method does not increase local computation.
In settings \emph{different} from the one considered in this paper, communication-efficient approaches have been recently studied with triggered communication protocols \citep{liu2017,lan2017}.
Except for convergence guarantees however, no theoretical justification for communication reduction has been established in \citep{liu2017}.
While a sublinear convergence rate can be achieved by algorithms in \citep{lan2017}, the proposed gradient selection rule is nonadaptive and requires double-loop iterations.

\subsection{Our contributions}
Before introducing our approach, we revisit the popular GD method for \eqref{opt0} in the setting of one parameter server and $M$ workers: At iteration $k$, the server broadcasts the current model $\bbtheta^k$ to \emph{all} the workers; every worker $m\in{\cal M}$ computes $\nabla{\cal L}_m\big(\bbtheta^k\big)$ and uploads it to the server; and once receiving gradients from all workers, the server updates the model parameters via
\begin{flalign}\label{eq.gd1}
&{\rm \textbf{GD iteration}}\qquad\qquad\qquad\bbtheta^{k+1}=\bbtheta^k-\alpha\nabla_{\rm GD}^k~~~{\rm with}~~~\nabla_{\rm GD}^k:=\! \sum_{m\in{\cal M}}\nabla{\cal L}_m\big(\bbtheta^k\big)&
\end{flalign}
where $\alpha$ is a stepsize, and $\nabla_{\rm GD}^k$ is an aggregated gradient that summarizes the model change.
To implement \eqref{eq.gd1}, the server has to communicate with \emph{all} workers to obtain fresh $\{\nabla{\cal L}_m\big(\bbtheta^k\big)\}$.

In this context, the present paper puts forward a new batch gradient method (as simple as GD) that can \emph{skip} communication at certain rounds, which justifies the term \textbf{L}azily \textbf{A}ggregated \textbf{G}radient (\textbf{LAG}). With its derivations deferred to Section \ref{sec.tgas}, LAG resembles \eqref{eq.gd1}, given by
\begin{flalign}\label{eq.LAG1-0}
&{\rm \textbf{LAG iteration}}\qquad\qquad\quad~~\bbtheta^{k+1}=\bbtheta^k-\alpha\nabla^k~~~~{\rm with}~~~~\nabla^k:=\! \sum_{m\in{\cal M}}\nabla{\cal L}_m\big(\hat{\bbtheta}_m^k\big)&
\end{flalign}
where each $\nabla{\cal L}_m(\hat{\bbtheta}_m^k)$ is either $\nabla{\cal L}_m(\bbtheta^k)$, when $\hat{\bbtheta}_m^k=\bbtheta^k$, or an outdated gradient that has been computed using an old copy $\hat{\bbtheta}_m^k\neq \bbtheta^k$.
Instead of requesting fresh gradient from every worker in \eqref{eq.gd1}, the twist is to obtain $\nabla^k$ by refining the previous aggregated gradient $\nabla^{k-1}$; that is, using only the new gradients from the \emph{selected} workers in ${\cal M}^k$, while reusing the outdated gradients from the rest of workers.
Therefore, with $\hat{\bbtheta}_m^k\!:=\!\bbtheta^k,\,\forall m\!\in\!{\cal M}^k,~\hat{\bbtheta}_m^k\!:=\!\hat{\bbtheta}_m^{k-1}\!\!,\,\forall m\!\notin\!{\cal M}^k$, LAG in \eqref{eq.LAG1-0} is equivalent to
\begin{flalign}\label{eq.LAG1}
&{\rm \textbf{LAG iteration}}\qquad\qquad\quad~~\bbtheta^{k+1}=\bbtheta^k-\alpha\nabla^k~~~{\rm with}~~~\nabla^k\!=\!\nabla^{k-1}\!+\!\!\!\sum_{m\in{\cal M}^k}\delta\nabla^k_m&
\end{flalign}
where $\delta\nabla^k_m:=\nabla{\cal L}_m(\bbtheta^k)\!-\!\nabla{\cal L}_m(\hat{\bbtheta}_m^{k-1})$ is the difference between two evaluations of $\nabla{\cal L}_m$ at the current iterate $\bbtheta^k$ and the old copy $\hat{\bbtheta}_m^{k-1}$.
If $\nabla^{k-1}$ is stored in the server, this simple modification scales down the number of communication rounds per iteration from GD's $M$ to LAG's $|{\cal M}^k|$.

We develop two different rules to select ${\cal M}^k$. The first rule is adopted by the parameter server (PS), and the second one by every worker (WK). At iteration $k$,

\vspace{0.1cm}
\noindent\textbf{LAG-PS}: the server determines ${\cal M}^k$ and sends $\bbtheta^k$ to the workers in ${\cal M}^k$; each worker $m\!\in\!{\cal M}^k$ computes $\nabla\!{\cal L}_m(\bbtheta^k)$ and uploads $\delta\nabla^k_m$; workers in ${\cal M}^k$ do nothing; the server updates via \eqref{eq.LAG1};

\vspace{0.1cm}
\noindent\textbf{LAG-WK}: the server broadcasts $\bbtheta^k$ to all workers; every worker computes $\nabla{\cal L}_m(\bbtheta^k)$, and checks if it belongs to ${\cal M}^k$; only the workers in ${\cal M}^k$ upload $\delta\nabla^k_m$; the server updates via \eqref{eq.LAG1}.
\vspace{0.1cm}

See a comparison of two LAG variants with GD in Table \ref{tab:3algcomp}.

\begin{wrapfigure}{r}{3.1in}
\vspace{-0.5cm}
\def\epsfsize#1#2{0.35#1}
\centerline{\epsffile{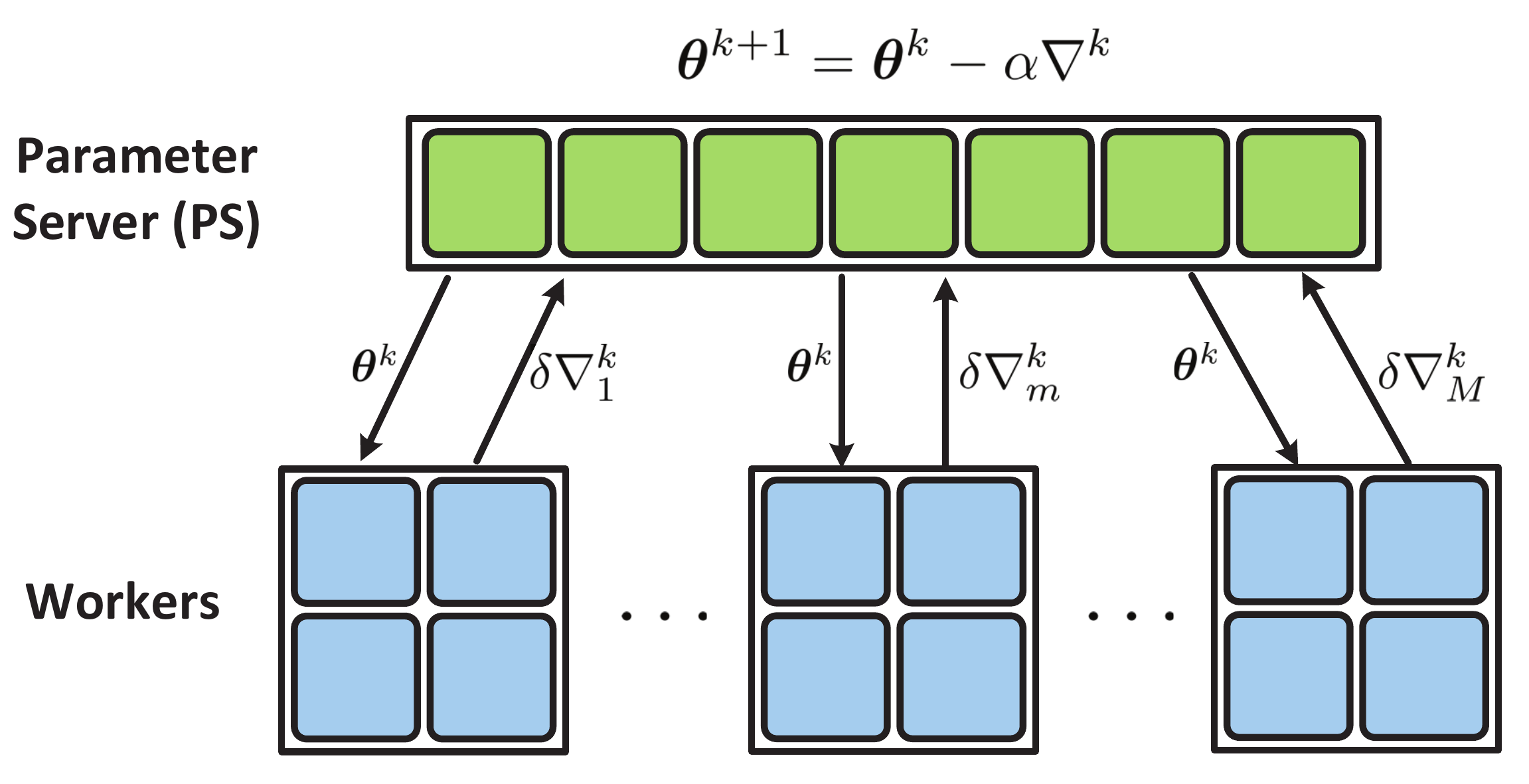}}
\vspace*{-6pt}
  \caption{LAG in a parameter server setup.}
\label{fig:pswk-diag}
\vspace{-0.3cm}
\end{wrapfigure}
Naively reusing outdated gradients, while saving communication per iteration, can increase the total number of iterations.
To keep this number in control, we judiciously design our simple trigger rules so that LAG can: i) achieve the \emph{same} order of convergence rates (thus iteration complexities) as batch GD under strongly-convex, convex, and nonconvex smooth cases; and, ii) require \emph{reduced} communication to achieve a targeted learning accuracy, when the distributed datasets are heterogeneous (measured by certain quantity specified later).
In certain learning settings, LAG requires only ${\cal O}(1/M)$ communication of GD.
 Empirically, we found that LAG can reduce the communication required by GD and other distributed parallel learning methods by several orders of magnitude.

\vspace{0.1cm}
\noindent\textbf{Notation}. Bold lowercase letters denote
column vectors, which are transposed by $(\cdot)^{\top}$. And $\|\mathbf{x}\|$ denotes the $\ell_2$-norm of $\mathbf{x}$. Inequalities for vectors $\mathbf{x} > \mathbf{0}$ is defined entrywise.

 \begin{table}
\small
\vspace{-0.2cm}
\centering
 \begin{tabular}[t]{ c || c | c ||c |c|| c |c  }
 \hline\hline \textbf{Metric} &\multicolumn{2}{|c||}{\textbf{Communication}}& \multicolumn{2}{|c||}{\textbf{Computation}}& \multicolumn{2}{|c}{\textbf{Memory}}\\
\hline \!\textbf{Algorithm}\!&\textbf{PS$\rightarrow$WK}~$m$\!\!&\!\textbf{WK} $m$ $\rightarrow$\textbf{PS}\!& ~~~~\textbf{PS}~~~~~&~~\textbf{WK} $m$~~~~& ~~~~\textbf{PS}~~~~&~\textbf{WK}~$m$\\ \hline \hline
    \textbf{GD} & $\bbtheta^k$ &  $\nabla {\cal L}_m$ & \eqref{eq.gd1} & $\nabla {\cal L}_m$ &  $\bbtheta^k$& $/$ \\ \hline
     \textbf{LAG-PS} & \!\!\!$\bbtheta^k$\!\!, if $m\!\in\!{\cal M}^k$\!\!\!\! &\!\!\!$\delta\nabla^k_m$, if $m\!\in\!{\cal M}^k$\!\!\!  & \!\!\! \eqref{eq.LAG1}, \eqref{eq.trig-cond2}\!\! & \!\!\!$\nabla\! {\cal L}_m$, if $m\!\in\!{\cal M}^k$\!\!\! &  \!\!\!\! $\bbtheta^k\!,\nabla^k\!,\{\hat{\bbtheta}_m^k\}$\!\!\!& \!\!\!\!\!\! $\nabla {\cal L}_m(\hat{\bbtheta}_m^k)$\!\!\!\!\!  \\\hline
    \textbf{LAG-WK} & $\bbtheta^k$ & \!\!\!$\delta\nabla^k_m$, if $m\!\in\!{\cal M}^k$\!\!\! & \eqref{eq.LAG1}  &\!\!\!$\nabla {\cal L}_m,\eqref{eq.trig-cond1}$\!\!\!&  $\bbtheta^k,\nabla^k$&\!\!\!\!\! $\nabla {\cal L}_m(\hat{\bbtheta}_m^k)$\!\!\!\!\\ \hline\hline
    \end{tabular}
    \vspace{-0.1cm}
\caption{A comparison of communication, computation and memory requirements.
\textbf{PS} denotes the parameter server, \textbf{WK} denotes the worker,
\textbf{PS$\rightarrow$WK $m$} is the communication link from the server to the worker $m$, and \textbf{WK} $m$ $\rightarrow$ \textbf{PS} is the communication link from the worker $m$ to the server.}
\label{tab:3algcomp}
\vspace{-0.2cm}
\end{table}

\section{LAG: Lazily Aggregated Gradient Approach}
\label{sec.tgas}

In this section, we formally develop our LAG method, and present the intuition and basic principles behind its design.
The original idea of LAG comes from a simple rewriting of the GD iteration \eqref{eq.gd1} as
\begin{equation}\label{eq.gd2}
	\bbtheta^{k+1}=\bbtheta^k-\alpha \sum_{m\in{\cal M}}\nabla{\cal L}_m(\bbtheta^{k-1}) -\alpha\sum_{m\in{\cal M}}\left(\nabla{\cal L}_m\big(\bbtheta^k\big)-\nabla{\cal L}_m\big(\bbtheta^{k-1}\big)\right).
\end{equation}
Let us view $\nabla{\cal L}_m(\bbtheta^k)\!-\!\nabla{\cal L}_m(\bbtheta^{k-1})$ as a refinement to $\nabla{\cal L}_m(\bbtheta^{k-1})$, and recall that obtaining this refinement requires a round of communication between the server and the worker $m$.
Therefore, to save communication, we can skip the server's communication with the worker $m$ if this refinement is small compared to the old gradient; that is, $\|\nabla{\cal L}_m(\bbtheta^k)-\nabla{\cal L}_m(\bbtheta^{k-1})\|\ll \|\sum_{m\in{\cal M}}\nabla{\cal L}_m(\bbtheta^{k-1})\|$.

Generalizing on this intuition, given the generic outdated gradient components $\{\nabla{\cal L}_m(\hat{\bbtheta}_m^{k-1})\}$ with $\hat{\bbtheta}_m^{k-\!1}\!\!=\!\bbtheta_m^{k-1-\tau_m^{k-1}}\!\!$ for a certain $\tau_m^{k-1}\!\geq\! 0$, if communicating with some workers will bring only small gradient refinements, we skip those communications (contained in set ${\cal M}^k_c$) and end up with
\begin{subequations}\label{eq.LAG2}
	\begin{align}
	\bbtheta^{k+1}&=\bbtheta^k-\alpha \sum_{m\in{\cal M}}\nabla{\cal L}_m\big(\hat{\bbtheta}_m^{k-1}\big)-\alpha\sum_{m\in{{\cal M}^k}}\left(\nabla{\cal L}_m\big(\bbtheta^k\big)-\nabla{\cal L}_m\big(\hat{\bbtheta}_m^{k-1}\big)\right)\label{eq.LAG2-1}\\
	&=\bbtheta^k-\alpha \nabla {\cal L}(\bbtheta^k) -\alpha\sum_{m\in{\cal M}^k_c}\left(\nabla{\cal L}_m\big(\hat{\bbtheta}_m^{k-1}\big)-\nabla{\cal L}_m\big(\bbtheta^k\big)\right)\label{eq.LAG2-2}
\end{align}
\end{subequations}
where ${\cal M}^k$ and ${\cal M}^k_c$ are the sets of workers that \emph{do} and \emph{do not} communicate with the server, respectively.
It is easy to verify that \eqref{eq.LAG2} is identical to \eqref{eq.LAG1-0} and \eqref{eq.LAG1}.
Comparing \eqref{eq.gd1} with \eqref{eq.LAG2-2}, when ${\cal M}^k_c$ includes more workers, more communication is saved, but $\bbtheta^k$ is updated by a coarser gradient.

Key to addressing this communication versus accuracy tradeoff is a principled criterion to select a subset of workers ${\cal M}^k_c$ that do not communicate with the server at each round.
To achieve this ``sweet spot,'' we will rely on the fundamental descent lemma.
For GD, it is given as follows \citep{nesterov2013}.

\vspace{-0.2cm}
\begin{lemma}[GD descent in objective]
\label{lemma5}
Suppose ${\cal L}(\bbtheta)$ is $L$-smooth, and $\bar{\bbtheta}^{k+1}$ is generated by running one-step GD iteration \eqref{eq.gd1} given $\bbtheta^k$ and stepsize $\alpha$. Then the objective values satisfy
\begin{equation}\label{eq.lemma5}
 {\cal L}(\bar{\bbtheta}^{k+1})- {\cal L}(\bbtheta^k)\leq -\left(\alpha-\frac{\alpha^2 L}{2}\right)\|\nabla {\cal L}(\bbtheta^k)\|^2:=\Delta_{\rm GD}^k(\bbtheta^k).
\end{equation}
\end{lemma}
\vspace{-0.2cm}

Likewise, for our wanted iteration \eqref{eq.LAG2}, the following holds; its proof is given in the Supplement.

\vspace{-0.2cm}
\begin{lemma}[LAG descent in objective]
\label{lemma1}
Suppose ${\cal L}(\bbtheta)$ is $L$-smooth, and $\bbtheta^{k+1}$ is generated by running one-step LAG iteration \eqref{eq.LAG1} given $\bbtheta^k$. The objective values satisfy (cf. $\delta\nabla^k_m$ in \eqref{eq.LAG1})
\begin{equation}\label{eq.lemma1}
\small
{\cal L}(\bbtheta^{k+1})\!-\!{\cal L}(\bbtheta^k)\leq \!-\frac{\alpha}{2}\left\|\nabla{\cal L}(\bbtheta^k)\right\|^2+\frac{\alpha}{2}\Big\|\!\sum_{m\in{\cal M}^k_c}\!\!\!\delta\nabla^k_m\Big\|^2\!+\!\left(\frac{L}{2}\!-\!\frac{1}{2\alpha}\right)\!\left\|\bbtheta^{k+1}\!\!-\!\bbtheta^k\right\|^2\!\!:=\!\Delta_{\rm LAG}^k(\bbtheta^k).\!
	\end{equation}
\end{lemma}
\vspace{-0.2cm}

Lemmas \ref{lemma5} and \ref{lemma1} estimate the objective value descent by performing one-iteration of the GD and LAG methods, respectively, conditioned on a common iterate $\bbtheta^k$.
GD finds {\small$\Delta_{\rm GD}^k(\bbtheta^k)$} by performing $M$ rounds of communication with all the workers, while LAG yields {\small$\Delta_{\rm LAG}^k(\bbtheta^k)$} by performing only $|{\cal M}^k|$ rounds of communication with a selected subset of workers.
Our pursuit is to select ${\cal M}^k$ to ensure that \emph{LAG enjoys larger per-communication descent than GD}; that is
\begin{equation}\label{eq.cond}
\frac{\Delta_{\rm LAG}^k(\bbtheta^k)}{|{\cal M}^k|}\leq \frac{\Delta_{\rm GD}^k(\bbtheta^k)}{M}.
\end{equation}

If we choose the standard $\alpha=1/L$ in Lemmas \ref{lemma5} and \ref{lemma1}, it follows that
\begin{subequations}\label{eq.descomp}
	\begin{empheq}[box=\fbox]{align}
	~~\Delta_{\rm GD}^k(\bbtheta^k) &:=-\frac{1}{2L}\left\|\nabla {\cal L}(\bbtheta^k)\right\|^2\qquad\qquad\qquad\qquad\qquad\qquad\qquad\qquad\qquad  \label{eq.des-gd}\\
	~~\Delta_{\rm LAG}^k(\bbtheta^k)&:=-\frac{1}{2L}\left\|\nabla{\cal L}(\bbtheta^k)\right\|^2\!+\frac{1}{2L}\Bigg\|\!\sum_{m\in{\cal M}_c^k}\!\Big(\nabla{\cal L}_m\big(\hat{\bbtheta}_m^{k-1}\big)-\nabla{\cal L}_m\big(\bbtheta^k\big)\!\Big)\Bigg\|^2.  ~~                 \label{eq.des-fiag}
\end{empheq}
\end{subequations}

Plugging \eqref{eq.descomp} into \eqref{eq.cond}, and rearranging terms, \eqref{eq.cond} is equivalent to
\begin{equation}\label{eq.cond-2}
\Bigg\|\sum_{m\in{\cal M}_c^k}\!\!\Big(\nabla{\cal L}_m\big(\hat{\bbtheta}_m^{k-1}\big)\!-\!\nabla{\cal L}_m\big(\bbtheta^k\big)\!\Big)\Bigg\|^2\leq \left|{\cal M}_c^k\right| \left\|\nabla {\cal L}(\bbtheta^k)\right\|^2\Big/M.
\end{equation}

Note that since we have
\begin{equation}\label{eq.cond-3}
\Bigg\|\!\sum_{m\in{\cal M}_c^k}\!\Big(\nabla{\cal L}_m\big(\hat{\bbtheta}_m^{k-1}\big)\!-\!\nabla{\cal L}_m\big(\bbtheta^k\big)\!\Big)\Bigg\|^2\leq \left|{\cal M}_c^k\right|	\sum_{m\in{\cal M}_c^k}\left\|\nabla{\cal L}_m\big(\hat{\bbtheta}_m^{k-1}\big)\!-\!\nabla{\cal L}_m\big(\bbtheta^k\big)\right\|^2
\end{equation}
if we can further show that
\begin{equation}\label{eq.cond-4}
\left\|\nabla{\cal L}_m\big(\hat{\bbtheta}_m^{k-1}\big)\!-\!\nabla{\cal L}_m\big(\bbtheta^k\big)\right\|^2\leq \left\|\nabla {\cal L}(\bbtheta^k)\right\|^2\Big/{M^2},~~~\forall m\in{\cal M}^k_c.
\end{equation}
then we can prove that \eqref{eq.cond-2} holds thus \eqref{eq.cond} also holds.

However, directly checking \eqref{eq.cond-4} at each worker is expensive since i) obtaining $\|\nabla {\cal L}(\bbtheta^k)\|^2$ requires information from all the workers; and ii) each worker does not know ${\cal M}_c^k$.
Instead, we approximate $\|\nabla {\cal L}(\bbtheta^k)\|^2$ in \eqref{eq.cond-4} by
\begin{equation}\label{eq.delta-approx}
\left\|\nabla {\cal L}(\bbtheta^k)\right\|^2\approx \frac{1}{\alpha^2}\sum_{d=1}^D \xi_d\Big\|\bbtheta^{k+1-d}-\bbtheta^{k-d}\Big\|^2
\end{equation}
where $\{\xi_d\}_{d=1}^D$ are constant weights.
The rationale here is that, as ${\cal L}$ is smooth, $\nabla {\cal L}(\bbtheta^k)$ cannot be very different from the recent gradients or the recent iterate \emph{lags}.

Building upon \eqref{eq.cond-4} and \eqref{eq.delta-approx}, we will
 include worker $m$ in ${\cal M}^k_c$ of \eqref{eq.LAG2} if it satisfies
\begin{subequations}\label{eq.trig-cond}
		\begin{flalign}\label{eq.trig-cond1}
&{\normalsize \rm \textbf{LAG-WK condition}}~~~~~~\left\|\nabla {\cal L}_m(\hat{\bbtheta}_m^{k-1})\!-\!\nabla {\cal L}_m(\bbtheta^k)\right\|^2\!\!\leq\! \frac{1}{\alpha^2M^2}\sum_{d=1}^D \xi_d\!\left\|\bbtheta^{k+1-d}\!-\!\bbtheta^{k-d}\right\|^2\!\!.\!\!\!\!\!&
	\end{flalign}
{\normalsize Condition \eqref{eq.trig-cond1} is checked at \emph{the worker side} after each worker receives $\bbtheta^k$ from the server and  computes its {\small$\nabla {\cal L}_m(\bbtheta^k)$}.
If broadcasting is also costly, we can resort to the following \emph{server side} rule:}	
\begin{flalign}\label{eq.trig-cond2}
&{\normalsize \rm \textbf{LAG-PS condition}}~~~~~~~~~L_m^2\left\|\hat{\bbtheta}_m^{k-1}-\bbtheta^k\right\|^2\leq \frac{1}{\alpha^2M^2}\sum_{d=1}^D \xi_d\left\|\bbtheta^{k+1-d}-\bbtheta^{k-d}\right\|^2.&
	\end{flalign}
\end{subequations}
The values of $\{\xi_d\}$ and $D$ admit simple choices, e.g., $\xi_d=1/D,\,\forall d$ with $D=10$ used in the simulations.

\begin{table}
\vspace{-0.6cm}
    \begin{tabular}{c c}
    \hspace{-0.4cm}
\begin{minipage}[t]{7.5cm}
  \vspace{0pt}
  \begin{algorithm}[H]
  \small
    \caption{LAG-WK}\label{algo:f-iag}
	\begin{algorithmic}[1]
		\State\textbf{Input:}~Stepsize $\alpha>0$, and $\{\xi_d\}$.
		\State\textbf{Initialize:}~$\bbtheta^1, \{\nabla {\cal L}_m(\hat{\bbtheta}^0_m),\,\forall m\}$.
		\For {$k= 1, 2,\ldots, K$}
		\State Server \textbf{broadcasts} $\bbtheta^k$ to all workers.
	    \For {worker $m=1, \ldots, M$}
		\State Worker $m$ \textbf{computes} $\nabla {\cal L}_m(\bbtheta^k)$.
	    \State Worker $m$ \textbf{checks} condition \eqref{eq.trig-cond1}.	
		\If {worker $m$ violates \eqref{eq.trig-cond1}}
		\State Worker $m$ \textbf{uploads} $\delta\nabla^k_m$.\\
		\qquad\qquad\Comment{Save $\nabla {\cal L}_m(\hat{\bbtheta}_m^k)=\nabla {\cal L}_m(\bbtheta^k$)}
		\Else \State{Worker $m$ uploads nothing.}
		\EndIf
		\EndFor
		\State Server \textbf{updates} via \eqref{eq.LAG1}.
		\EndFor
	\end{algorithmic}
  \end{algorithm}
\end{minipage}
&
\begin{minipage}[t]{7.5cm}
  \vspace{0pt}
  \begin{algorithm}[H]
  \small
    \caption{LAG-PS}\label{algo:f-iag2}
	\begin{algorithmic}[1]
		\State\textbf{Input:}~Stepsize $\alpha>0$, $\{\xi_d\}$, and $L_m,\,\forall m$.
		\State\textbf{Initialize:}~$\bbtheta^1, \{\hat{\bbtheta}^0_m,\!\nabla {\cal L}_m(\hat{\bbtheta}^0_m),\forall m\}$.
		\For {$k= 1, 2,\ldots, K$}
	    \For {worker $m=1, \ldots, M$}
	    \State Server \textbf{checks} condition \eqref{eq.trig-cond2}.	
	    \If {worker $m$ violates \eqref{eq.trig-cond2}}
	    \State Server \textbf{sends} $\bbtheta^k$ to worker $m$.\\\qquad\qquad\qquad\Comment{Save $\hat{\bbtheta}_m^k=\bbtheta^k$ at server}
		\State Worker $m$ \textbf{computes} $\nabla\! {\cal L}_m(\bbtheta^k)$.
		\State Worker $m$ \textbf{uploads} $\delta\nabla^k_m$.		
		\Else \State{No actions at server and worker $m$.}
	    \EndIf
		\EndFor
		\State Server \textbf{updates} via \eqref{eq.LAG1}.
		\EndFor
	\end{algorithmic}
  \end{algorithm}
\end{minipage}
   \end{tabular}
   \vspace{0.1cm}
   \caption{A comparison of LAG-WK and LAG-PS.}
     \vspace{-0.2cm}
\end{table}

\vspace{0.1cm}
\noindent\textbf{LAG-WK vs LAG-PS}. To perform \eqref{eq.trig-cond1}, the server needs to broadcast the current model $\bbtheta^k$, and all the workers need to compute the gradient; while performing \eqref{eq.trig-cond2}, the server needs the estimated smoothness constant $L_m$ for all the local functions. On the other hand, as it will be shown in Section \ref{sec.cc-ana}, \eqref{eq.trig-cond1} and \eqref{eq.trig-cond2} lead to the same worst-case convergence guarantees.
In practice, however, the server-side condition is more conservative than the worker-side one at communication reduction, because the smoothness of ${\cal L}_m$ readily implies that satisfying \eqref{eq.trig-cond2} will necessarily satisfy \eqref{eq.trig-cond1}, but not vice versa.
Empirically, \eqref{eq.trig-cond1} will lead to a larger ${\cal M}^k_c$ than that of \eqref{eq.trig-cond2}, and thus extra communication overhead will be saved.
Hence, \eqref{eq.trig-cond1} and \eqref{eq.trig-cond2} can be chosen according to users' preferences. LAG-WK and LAG-PS are summarized as Algorithms \ref{algo:f-iag} and \ref{algo:f-iag2}.

\vspace{0.1cm}

Regarding our proposed LAG method, two remarks are in order.

\vspace{0.1cm}
\noindent\textbf{R1)}
With recursive update of the lagged gradients in \eqref{eq.LAG1} and the lagged iterates in \eqref{eq.trig-cond}, implementing LAG is as simple as GD; see Table \ref{tab:3algcomp}. Both empirically and theoretically, we will further demonstrate that using lagged gradients even reduces the overall delay by cutting down costly communication.
\vspace{0.1cm}

\noindent\textbf{R2)}
Compared with existing efforts for communication-efficient learning such as quantized gradient, Nesterov's acceleration, dual coordinate ascent and second-order methods, LAG is not orthogonal to all of them.
Instead, LAG can be combined with these methods to develop even more powerful learning schemes. Extension to the proximal LAG is also possible to cover nonsmooth regularizers.


\section{Iteration and communication complexity}
\label{sec.cc-ana}
In this section, we establish the convergence of LAG, under the following standard conditions.

\vspace{0.1cm}

\noindent\textbf{Assumption 1}:
Loss function \emph{${\cal L}_m(\bbtheta)$ is $L_m$-smooth, and ${\cal L}(\bbtheta)$ is $L$-smooth.}
\vspace{0.1cm}

\noindent\textbf{Assumption 2}:
\emph{${\cal L}(\bbtheta)$ is convex and coercive.}
\vspace{0.1cm}

\noindent\textbf{Assumption 3}:
\emph{${\cal L}(\bbtheta)$ is $\mu$-strongly convex, or generally, satisfies the Polyak-{\L}ojasiewicz (PL) condition with the constant $\mu$; that is, $2\mu({\cal L}(\bbtheta^k)-{\cal L}(\bbtheta^*))\leq \|\nabla{\cal L}(\bbtheta^k)\|^2$.}
\vspace{0.1cm}

Note that the PL condition in Assumption 3 is strictly weaker than the strongly convexity (or even convexity), and it is satisfied by a wider range of machine learning problems such as least squares for underdetermined linear systems and logistic regression; see details in \citep{karimi2016}. 
While the PL condition is sufficient for the subsequent linear convergence analysis, we will still use the strong convexity for the ease of understanding by a wide audience. 
	
The subsequent analysis critically builds on the following \textbf{Lyapunov function}:
 \begin{equation}\label{eq.Lyap}
 	\mathbb{V}^k:={\cal L}(\bbtheta^k)-{\cal L}(\bbtheta^*)+\sum_{d=1}^D \beta_d\left\|\bbtheta^{k+1-d}-\bbtheta^{k-d}\right\|^2
 \end{equation}
where $\bbtheta^*$ is the minimizer of \eqref{opt0}, and $\{\beta_d\}$ are constants that will be determined later.

We will start with the sufficient descent of our $\mathbb{V}^k$ in \eqref{eq.Lyap}.

\begin{lemma}[descent lemma]
\label{lemma2}
	Under Assumption 1, if $\alpha$ and $\{\xi_d\}$ are chosen properly, there exist constants $c_0,\cdots,c_D\geq 0$ such that
the Lyapunov function in \eqref{eq.Lyap} satisfies
\begin{equation}\label{eq.lemma2}
\mathbb{V}^{k+1}-\mathbb{V}^k\leq -c_0\left\|\nabla{\cal L}(\bbtheta^k)\right\|^2-\sum_{d=1}^D c_d\left\|\bbtheta^{k+1-d}\!-\!\bbtheta^{k-d}\right\|^2\!\!\!
\end{equation}
which implies the descent in our Lyapunov function, that is, ${\small \mathbb{V}^{k+1}\leq\mathbb{V}^k}$.
\end{lemma}
Lemma \ref{lemma2} is a generalization of GD's descent lemma.
As specified in the supplementary material, under properly chosen $\{\xi_d\}$, the stepsize $\alpha\in(0, 2/L)$ including $\alpha=1/L$ guarantees \eqref{eq.lemma2}, matching the stepsize region of GD.
With ${\cal M}^k={\cal M}$ and $\beta_d=0,\,\forall d$ in \eqref{eq.Lyap}, Lemma \ref{lemma2} reduces to Lemma \ref{lemma5}.

\subsection{Convergence in strongly convex case}\label{subsec.scc}
We first present the convergence under the smooth and strongly convex condition.

\begin{customthm}{1}[strongly convex case]\label{theorem2}
Under Assumptions 1 and 3, the iterates $\{\bbtheta^k\}$ generated by LAG-WK or LAG-PS satisfy
\begin{align}\label{eq.theorem2}
~{\cal L}\big(\bbtheta^K\big)-{\cal L}\big(\bbtheta^*\big)\leq \big(1-c(\alpha;\{\xi_d\})\big)^K \,\mathbb{V}^0
\end{align}
where $\bbtheta^*$ is the minimizer of ${\cal L}(\bbtheta)$ in \eqref{opt0}, and $c(\alpha;\{\xi_d\})\in(0,1)$ is a constant depending on $\alpha,\,\{\xi_d\}$ and
$\{\beta_d\}$ as well as the condition number $\kappa:=L/\mu$ that are specified in the supplementary material.
\end{customthm}

\noindent\textbf{Iteration complexity}.
The iteration complexity in its generic form is complicated since $c(\alpha;\{\xi_d\})$ depends on the choice of several parameters.
Specifically, if we choose the parameters as follows
\begin{equation}\label{eq.para-set}
\xi_1=\cdots=\xi_D:=\xi<\frac{1}{D},~~~~\alpha:=\frac{1-\sqrt{D\xi}}{L},~~~~\beta_1=\cdots=\beta_D:=\frac{D-d+1}{2\alpha \sqrt{D/\xi}}
\end{equation}
then, following Theorem \ref{theorem2}, the iteration complexity of LAG in this case is
\begin{equation}\label{iter.para-set}
\mathbb{I}_{\rm LAG}(\epsilon)= \frac{\kappa}{1-\sqrt{D\xi}}\log\left(\epsilon^{-1}\right).	
\end{equation}

The iteration complexity in \eqref{iter.para-set} is on the same order of GD's iteration complexity $\kappa \log (\epsilon^{-1})$, but has a worse constant.
This is the consequence of using a smaller stepsize in \eqref{eq.para-set} (relative to $\alpha=1/L$ in GD) to simplify the choice of other parameters.
Empirically, LAG with $\alpha=1/L$ can achieve almost the same empirical iteration complexity as GD; see Section \ref{sec.nums}.
Building on the iteration complexity, we study next the \emph{communication complexity} of LAG.
In the setting of our interest, we define the communication complexity as the \emph{total number of uploads} over all the workers needed to achieve accuracy $\epsilon$.
While the accuracy refers to the objective optimality error in the strongly convex case, it is considered as the gradient norm in general (non)convex cases.

\begin{wrapfigure}{r}{3.1in}
\vspace{-0.1cm}
\def\epsfsize#1#2{0.295#1}
\centerline{\epsffile{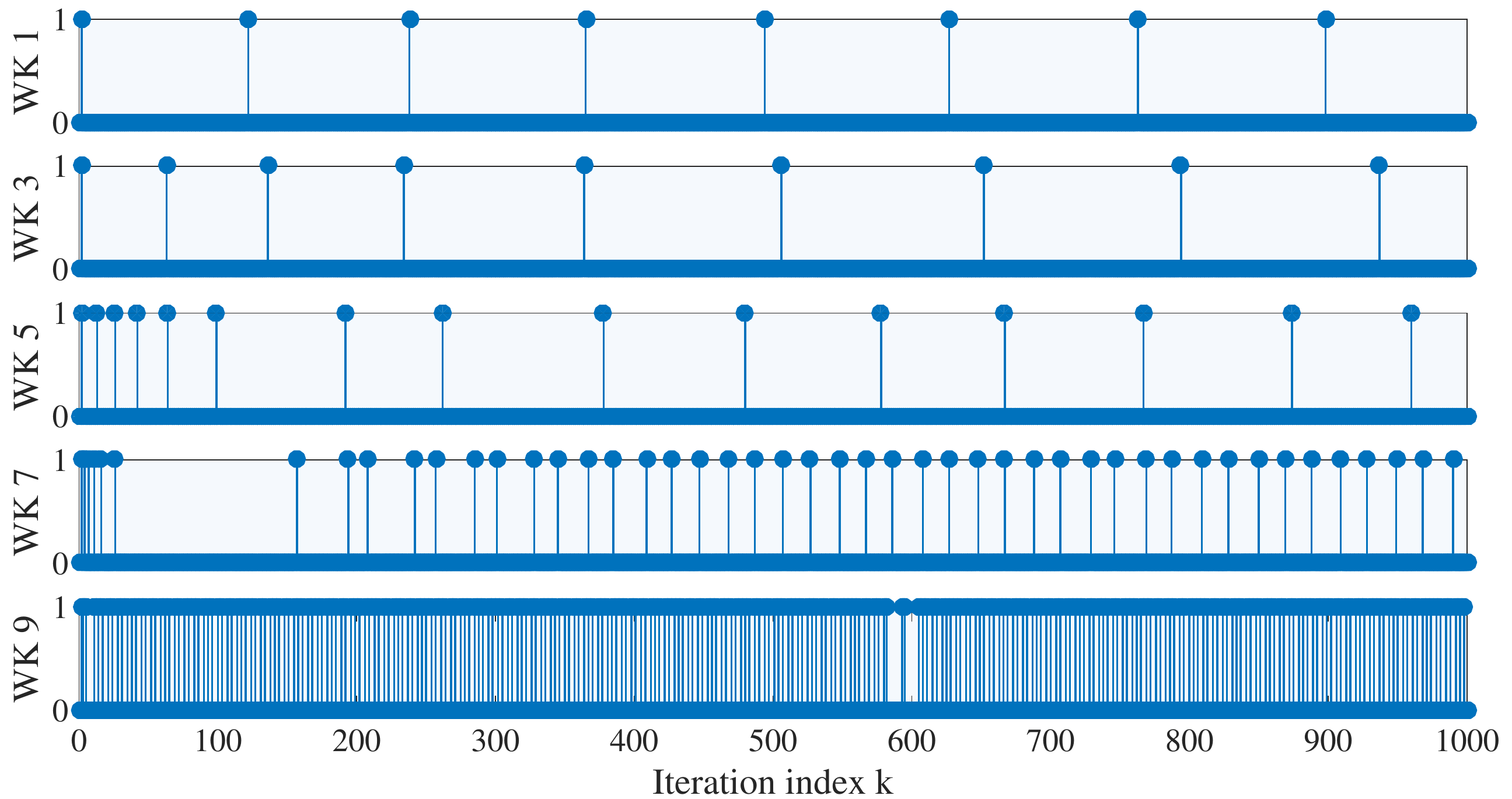}}
\vspace*{-8pt}
\captionsetup{format=plain}
  \caption{Communication events of workers $1,3,5,7,9$ over $1,000$ iterations. Each stick is an upload.
  An example with $L_1< \ldots <L_9$.}
\label{fig:commcheck}
\vspace{-0.4cm}
\end{wrapfigure}

The power of LAG is best illustrated by numerical examples; see an example of LAG-WK in Figure \ref{fig:commcheck}.
Clearly, workers with a small smoothness constant communicate with the server less frequently.
This intuition will be formally treated in the next lemma.
\begin{lemma}[lazy communication]\label{lemma4}
 Define the importance factor of every worker $m$ as $\mathds{H}(m):=L_m/L$.
 If the stepsize $\alpha$ and the constants $\{\xi_d\}$ in \eqref{eq.trig-cond} satisfy $\xi_D\leq \cdots \leq \xi_d \leq \cdots\leq\xi_1$ and worker $m$ satisfies
 \begin{equation}\label{eq.lemma4}
\mathds{H}^2(m)\leq \xi_d \big/ (d\alpha^2 L^2 M^2):=\gamma_d
 \end{equation}
 then, until iteration $k$, worker $m$ communicates with the server at most $k/(d+1)$ rounds.
\end{lemma}

Lemma \ref{lemma4} asserts that if the worker $m$ has a small $L_m$ (a close-to-linear loss function) such that $\mathds{H}^2(m)\leq \gamma_d$, then under LAG, it only communicates with the server at most $k/(d+1)$ rounds. This is in contrast to the total of $k$ communication rounds involved per worker under GD.
Ideally, we want as many workers satisfying \eqref{eq.lemma4} as possible, especially when $d$ is large.

To quantify the overall communication reduction, we will rely on what we term the \textbf{heterogeneity score function}, given by
	\begin{equation}\label{eq.heter-score}
		h(\gamma):=\frac{1}{M}\sum_{m\in{\cal M}}\mathds{1}(\mathds{H}^2(m)\leq \gamma)
	\end{equation}
	where the indicator $\mathds{1}$ equals $1$ when $\mathds{H}^2(m)\leq \gamma$ holds, and $0$ otherwise.
	Clearly, $h(\gamma)$ is a nondecreasing function of $\gamma$, that depends on the distribution of smoothness constants $L_1, L_2, \ldots, L_M$.
	It is also instructive to view it as the cumulative distribution function of the \emph{deterministic} quantity $\mathds{H}^2(m)$, implying $h(\gamma)\in[0,1]$.
Putting it in our context, the critical quantity $h(\gamma_d)$ lower bounds the fraction of workers that communicate with the server at most $k/(d+1)$ rounds until the $k$-th iteration.

We are now ready to present the communication complexity.

\begin{customprop}{1}[communication complexity]\label{prop.comm}
Under the same conditions as those in Theorem \ref{theorem2}, with $\gamma_d$ defined in \eqref{eq.lemma4} and the function $h(\gamma)$ defined in \eqref{eq.heter-score}, the communication complexity of LAG denoted as $\mathbb{C}_{\rm LAG}(\epsilon)$ is bounded by
\begin{equation}\label{eq.prop5}
\mathbb{C}_{\rm LAG}(\epsilon)\leq \bigg(\!1-\sum_{d=1}^D\left(\frac{1}{d}-\frac{1}{d+1}\right)h\left(\gamma_d\right)\!\!\bigg)\!M\,\mathbb{I}_{\rm LAG}(\epsilon):=\left(1-\Delta\bar{\mathbb{C}}(h;\{\gamma_d\})\right)\!M\,\mathbb{I}_{\rm LAG}(\epsilon)
\end{equation}
where the constant is defined as $\Delta\bar{\mathbb{C}}(h;\{\gamma_d\}):=\sum_{d=1}^D\big(\frac{1}{d}-\frac{1}{d+1}\big)h\left(\gamma_d\right)$.
\end{customprop}

The communication complexity in \eqref{eq.prop5} crucially depends on the iteration complexity $\mathbb{I}_{\rm LAG}(\epsilon)$ as well as what we call the \textbf{fraction of reduced communication per iteration} $\Delta\bar{\mathbb{C}}(h;\{\gamma_d\})$.
Simply choosing the parameters as \eqref{eq.para-set}, it follows from \eqref{iter.para-set} and \eqref{eq.prop5} that (cf. {\small$\gamma_d\!=\!{\xi}(1-\sqrt{D\xi})^{-2}M^{-2} d^{-1}$})
\begin{equation}\label{eq.prop5-2}
\mathbb{C}_{\rm LAG}(\epsilon)\leq \big(1-\Delta\bar{\mathbb{C}}(h;\xi)\big)\mathbb{C}_{\rm GD}(\epsilon)\big/\big(1-\sqrt{D\xi}\big).
\end{equation}
where the GD's complexity is $\mathbb{C}_{\rm GD}(\epsilon)=M\kappa \log (\epsilon^{-1})$.
In \eqref{eq.prop5-2}, due to the nondecreasing property of $h(\gamma)$, increasing the constant $\xi$ yields a smaller fraction of workers $1-\Delta\bar{\mathbb{C}}(h;\xi)$ that are communicating per iteration, yet with a larger number of iterations (cf. \eqref{iter.para-set}).
The key enabler of LAG's communication reduction is a heterogeneous environment associated with a favorable $h(\gamma)$ ensuring that the benefit of increasing $\xi$ is more significant than its effect on increasing iteration complexity.
More precisely, for a given $\xi$, if $h(\gamma)$ guarantees $\Delta\bar{\mathbb{C}}(h;\xi)>\sqrt{D\xi}$, then we have $\mathbb{C}_{\rm LAG}(\epsilon)<\mathbb{C}_{\rm GD}(\epsilon)$.
Intuitively speaking, if there is a large fraction of workers with small $L_m$,
LAG has lower communication complexity than GD.
An example follows to illustrate this reduction.

\vspace{0.1cm}
\noindent\textbf{Example}.
Consider $L_m=1,\, m\neq M$, and $L_M=L\geq M^2\gg 1$,
where we have $\mathds{H}(m)=1/L,m\neq M,\, \mathds{H}(M)=1$,
implying that $h(\gamma)\geq1-\frac{1}{M}$, if $\gamma\geq 1/L^2$.
Choosing $D\geq M$ and $\xi=M^2D/L^2<1/D$ in \eqref{eq.para-set} such that $\gamma_D\geq 1/L^2$ in \eqref{eq.lemma4}, we have (cf. \eqref{eq.prop5-2})
	\begin{equation}
\mathbb{C}_{\rm LAG}(\epsilon)\big/\mathbb{C}_{\rm GD}(\epsilon)\leq \left[1-\left(1-\frac{1}{D+1}\right)\!\!\left(1-\frac{1}{M}\right)\right]\Big/\Big(1-{MD}/L\Big)\approx \frac{M+D}{M(D+1)}\approx\frac{2}{M}.
	\end{equation}
Due to technical issues in the convergence analysis, the current condition on $h(\gamma)$ to ensure LAG's communication reduction is relatively restrictive.
Establishing communication reduction on a broader learning setting that matches the LAG's intriguing empirical performance is in our research agenda.

\subsection{Convergence in (non)convex case}
LAG's convergence and communication reduction guarantees go beyond the strongly-convex case.
We next establish the convergence of LAG for general convex functions.
\begin{customthm}{2}[convex case]\label{theorem1}
Under Assumptions 1 and 2, if $\alpha$ and $\{\xi_d\}$ are chosen properly, then the iterates $\{\bbtheta^k\}$ generated by LAG-WK or LAG-PS satisfy
\begin{equation}\label{eq.theorem1}
{\cal L}(\bbtheta^K)-{\cal L}(\bbtheta^*)={\cal O}\left({1}/{K}\right).
\end{equation}
\end{customthm}

For nonconvex objective functions, LAG can guarantee the following convergence result.
\begin{customthm}{3}[nonconvex case]\label{theorem0}
Under Assumption 1, if $\alpha$ and $\{\xi_d\}$ are chosen properly, then the iterates $\{\bbtheta^k\}$ generated by LAG-WK or LAG-PS satisfy
\begin{equation}\label{eq.theorem0}
\min_{1\leq k \leq K}\big\|\bbtheta^{k+1}-\bbtheta^k\big\|^2=o\left({1}/{K}\right)~~~{\rm and}~~~\min_{1\leq k \leq K}\big\|\nabla {\cal L}(\bbtheta^k)\big\|^2=o\left({1}/{K}\right).
\end{equation}
\end{customthm}

Theorems \ref{theorem1} and \ref{theorem0} assert that with the judiciously designed lazy gradient aggregation rules, LAG can achieve order of convergence rate identical to GD for general convex and nonconvex smooth objective functions. Furthermore, we next show that in these general cases, LAG still requires fewer communication rounds than GD, under certain conditions on the heterogeneity function $h(\gamma)$.

In the general smooth (possibly nonconvex) case however,
we define the communication complexity in terms of achieving $\epsilon$-gradient error; e.g., $\min_{k=1,\cdots,K}	 \|\nabla{\cal L}(\bbtheta^k)\|^2\leq \epsilon$. 
Similar to Proposition \ref{prop.comm}, we present the communication complexity as follows.

\begin{customprop}{2}[communication complexity]\label{propncvx}
Under Assumption 1, with $\Delta\bar{\mathbb{C}}(h;\{\gamma_d\})$ defined as in Proposition \ref{prop.comm}, the communication complexity of LAG denoted as $\mathbb{C}_{\rm N-LAG}(\epsilon)$ is bounded by
\begin{align}\label{prononconvex-r1}
\mathbb{C}_{\rm N-LAG}(\epsilon)\leq\left(1-\Delta\bar{\mathbb{C}}(h;\{\gamma_d\})\right)\!\frac{\mathbb{C}_{\rm N-GD}(\epsilon)}{(1-\sum_{d=1}^D\xi_d)}
\end{align}
where $\mathbb{C}_{\rm N-GD}(\epsilon)$ is the communication complexity of GD.
Choosing the parameters as \eqref{eq.para-set}, if the heterogeneity function $h(\gamma)$ satisfies that there exists $\gamma'$ such that $\gamma'<\frac{h(\gamma')}{(D+1)D M^2}$, then we have that
\begin{align}\label{prononconvex-r2}
\mathbb{C}_{\rm N-LAG}(\epsilon)<\mathbb{C}_{\rm N-GD}(\epsilon).
\end{align}
\end{customprop}

Along with Proposition \ref{prop.comm}, we have shown that for strongly convex, convex, and nonconvex smooth objective functions, LAG enjoys provably lower communication overhead relative to GD in certain heterogeneous learning settings.
In fact, the LAG's empirical performance gain over GD goes far beyond the above worst-case theoretical analysis, and lies in a much broader distributed learning setting, which is confirmed by the subsequent numerical tests.

\section{Numerical tests}
\label{sec.nums}
To validate the theoretical results, this section evaluates the empirical performance of LAG in linear and logistic regression tasks. All experiments were performed using MATLAB on an Intel CPU @ 3.4 GHz (32 GB RAM) desktop.
By default, we consider one server, and nine workers.
Throughout the test, we use the optimality error in objective ${\cal L}(\bbtheta^k)-{\cal L}(\bbtheta^*)$ as figure of merit of our solution.
To benchmark LAG, we consider the following approaches.

\noindent$\triangleright$ \textbf{Cyc-IAG} is the cyclic version of the incremental aggregated gradient (IAG) method \citep{blatt2007,gurbuzbalaban2017} that resembles the recursion \eqref{eq.LAG1}, but communicates with one worker per iteration in a cyclic fashion.

\noindent$\triangleright$ \textbf{Num-IAG} also resembles the recursion \eqref{eq.LAG1}, but it randomly selects one worker to obtain a fresh gradient per iteration with the probability of choosing worker $m$ equal to $L_m/\sum_{m\in{\cal M}}L_m$.

\noindent$\triangleright$ \textbf{Batch-GD} is the GD iteration \eqref{eq.gd1} that communicates with all the workers per iteration.

For LAG-WK, we choose $\xi_d=\xi={1}/{D}$ with $D=10$, and for LAG-PS, we choose more aggressive $\xi_d=\xi={10}/{D}$ with $D=10$.
Stepsizes for LAG-WK, LAG-PS, and GD are chosen as $\alpha=1/L$; to optimize performance and guarantee stability, stepsizes for Cyc-IAG and Num-IAG are chosen as $\alpha=1/(ML)$.
For the linear regression task, no regularization is added; for the logistic regression task, the $\ell_2$-regularization parameter is set to $\lambda=10^{-3}$.

\begin{figure}[t]
\vspace{-0.1cm}
\centering
\begin{tabular}{cc}
\includegraphics[width=6.5cm]{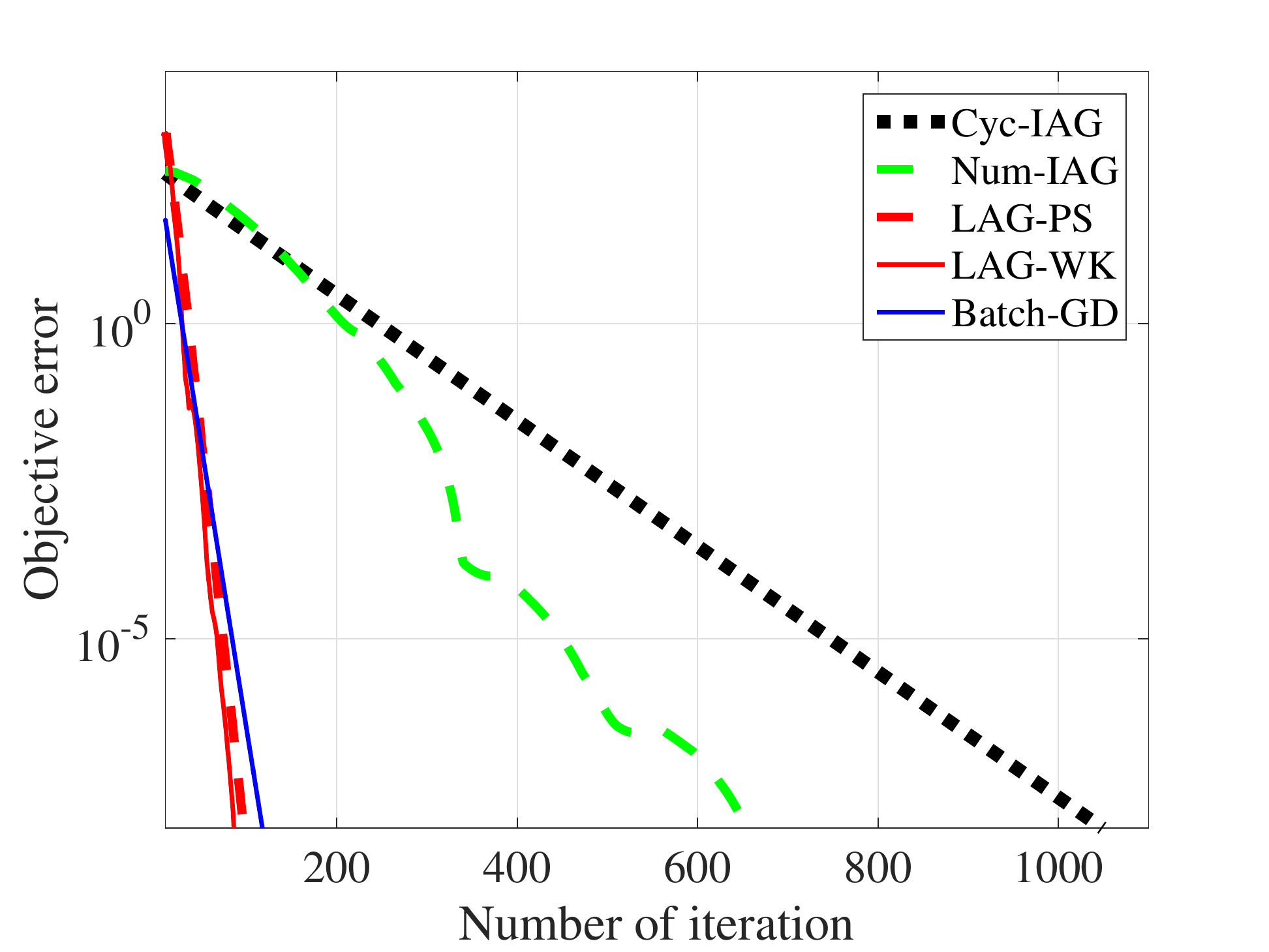}&
\includegraphics[width=6.5cm]{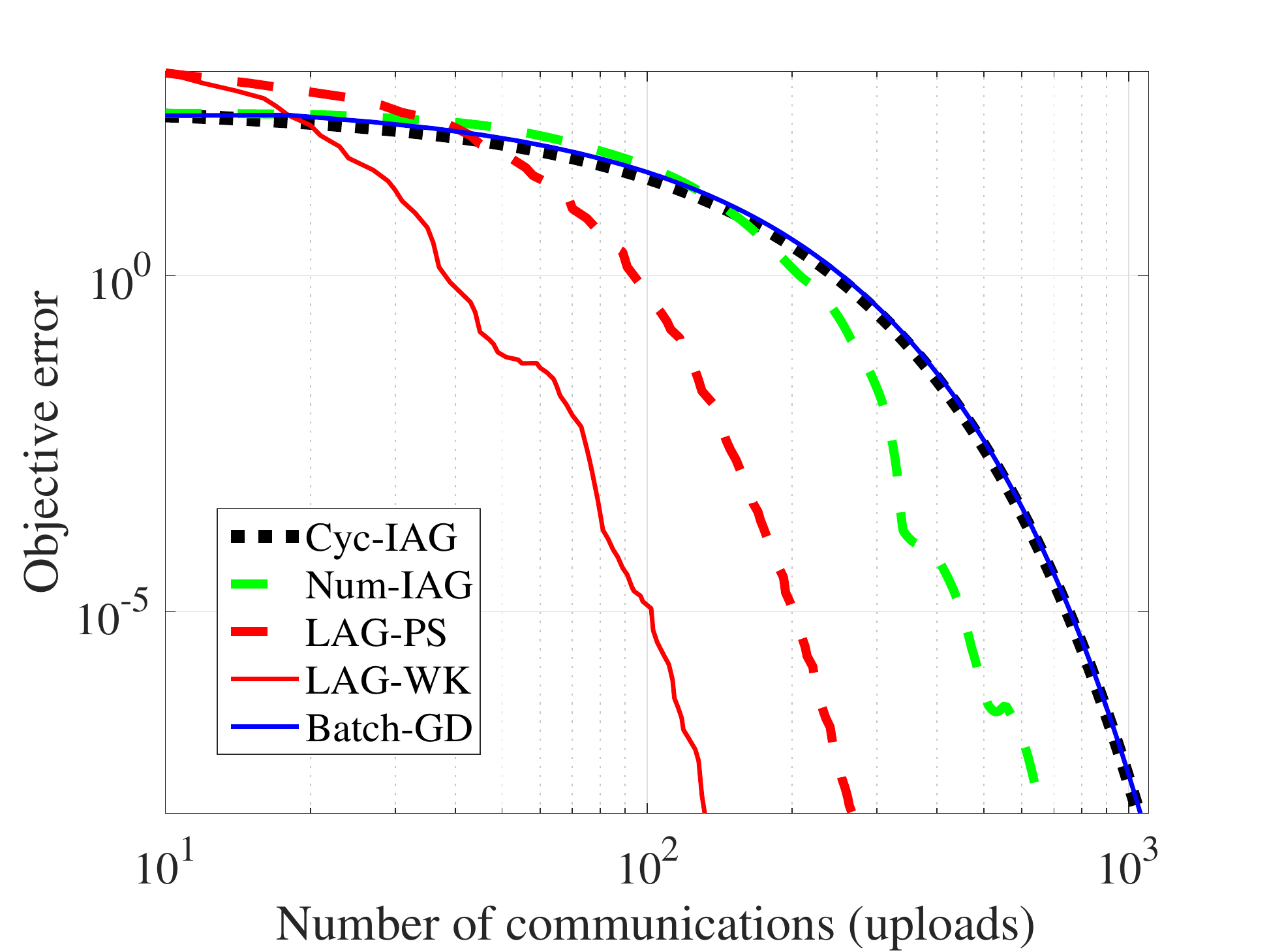}
\end{tabular}
\vspace*{-0.1cm}
  \caption{Iteration and communication complexity in synthetic datasets with increasing $L_m$.}
\label{fig:synfig1}
\vspace*{-0.2cm}
\end{figure}

\begin{figure}[t]
\centering
\begin{tabular}{cc}
\includegraphics[width=6.5cm]{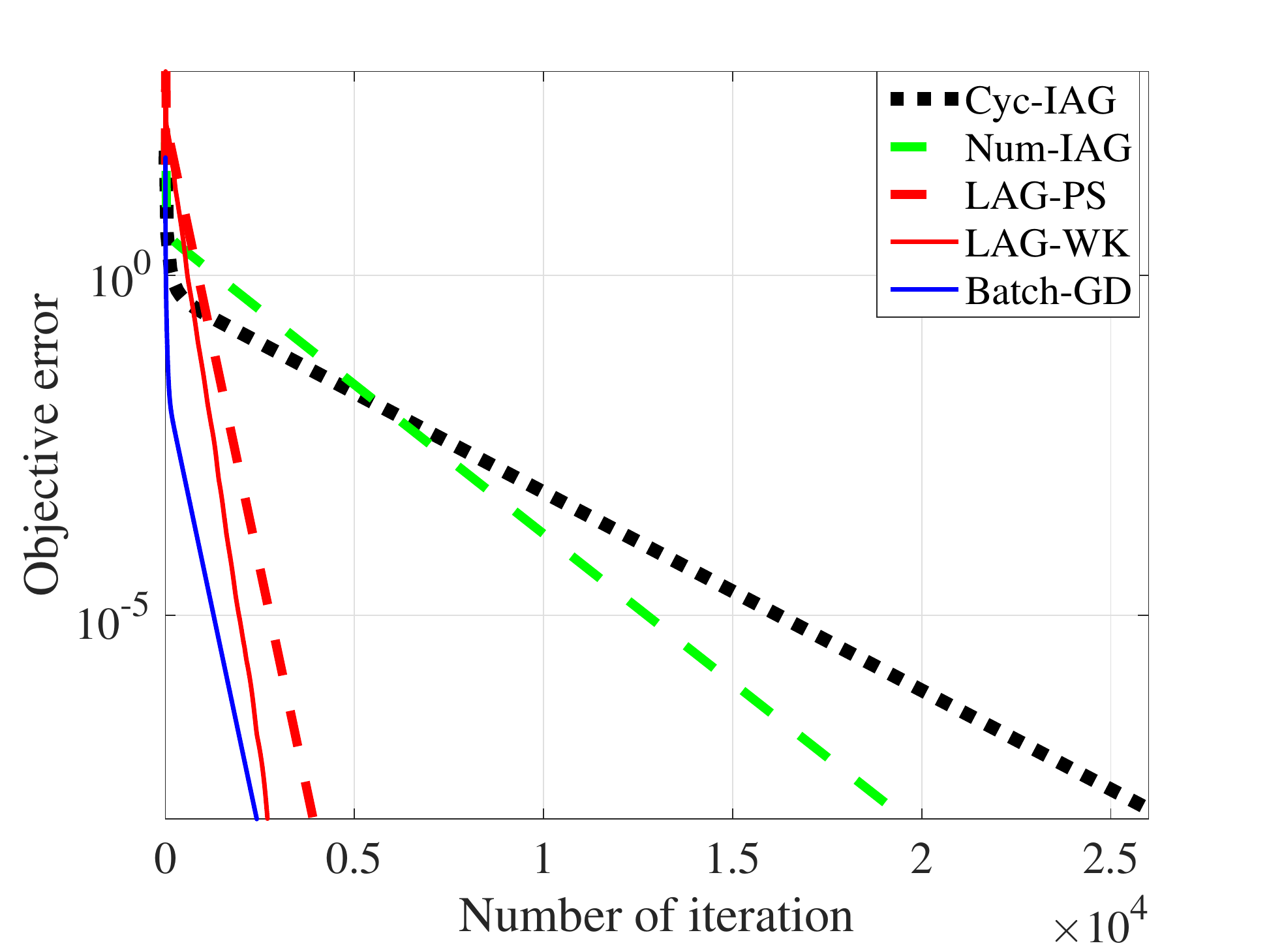}&
\includegraphics[width=6.5cm]{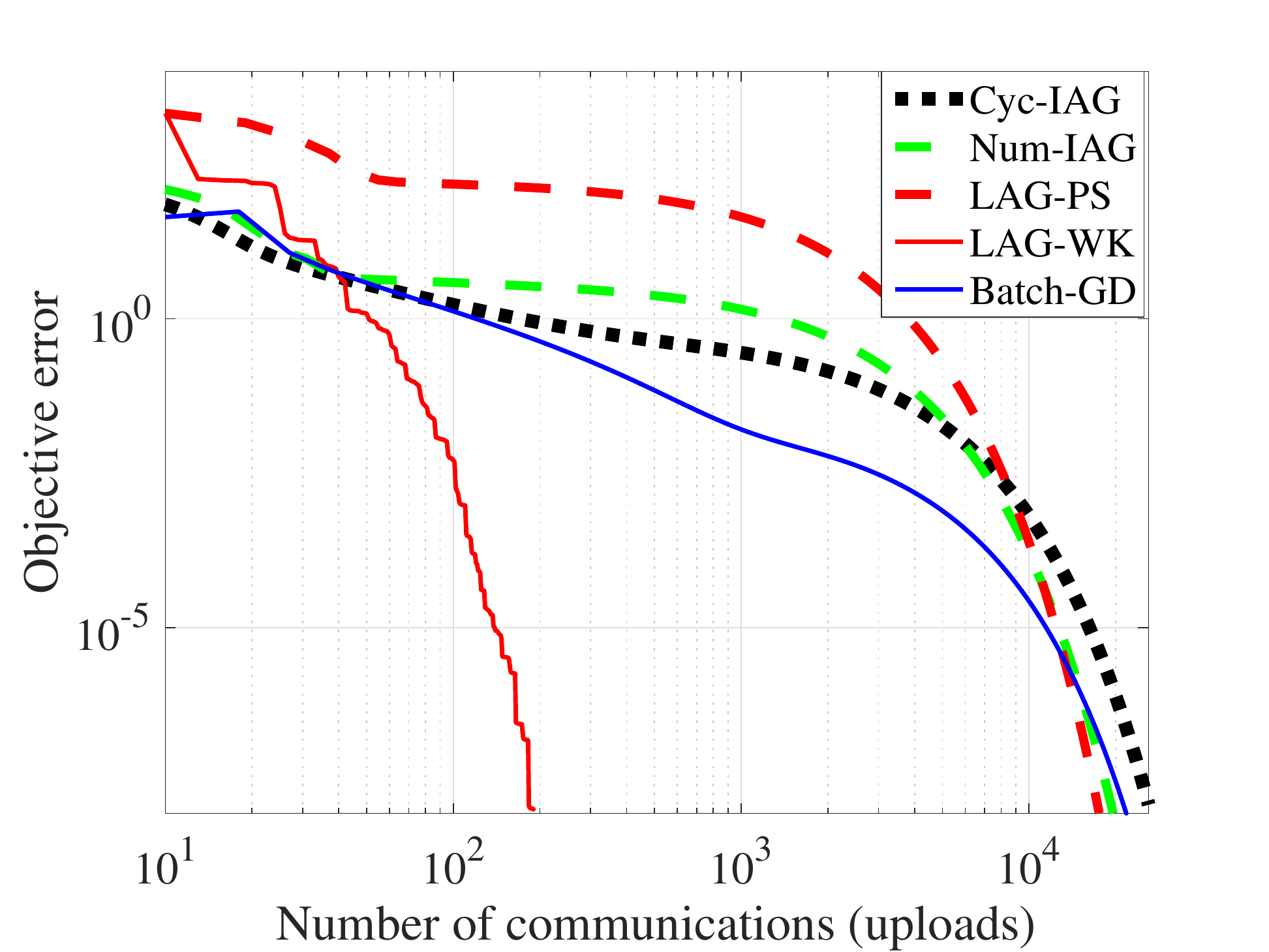}
\end{tabular}
\vspace*{-0.1cm}
  \caption{Iteration and communication complexity in synthetic datasets with uniform $L_m$.}
\label{fig:synfig2}
\vspace*{-0.2cm}
\end{figure}

\begin{figure}[t]
\centering
\begin{tabular}{cc}
\includegraphics[width=6.5cm]{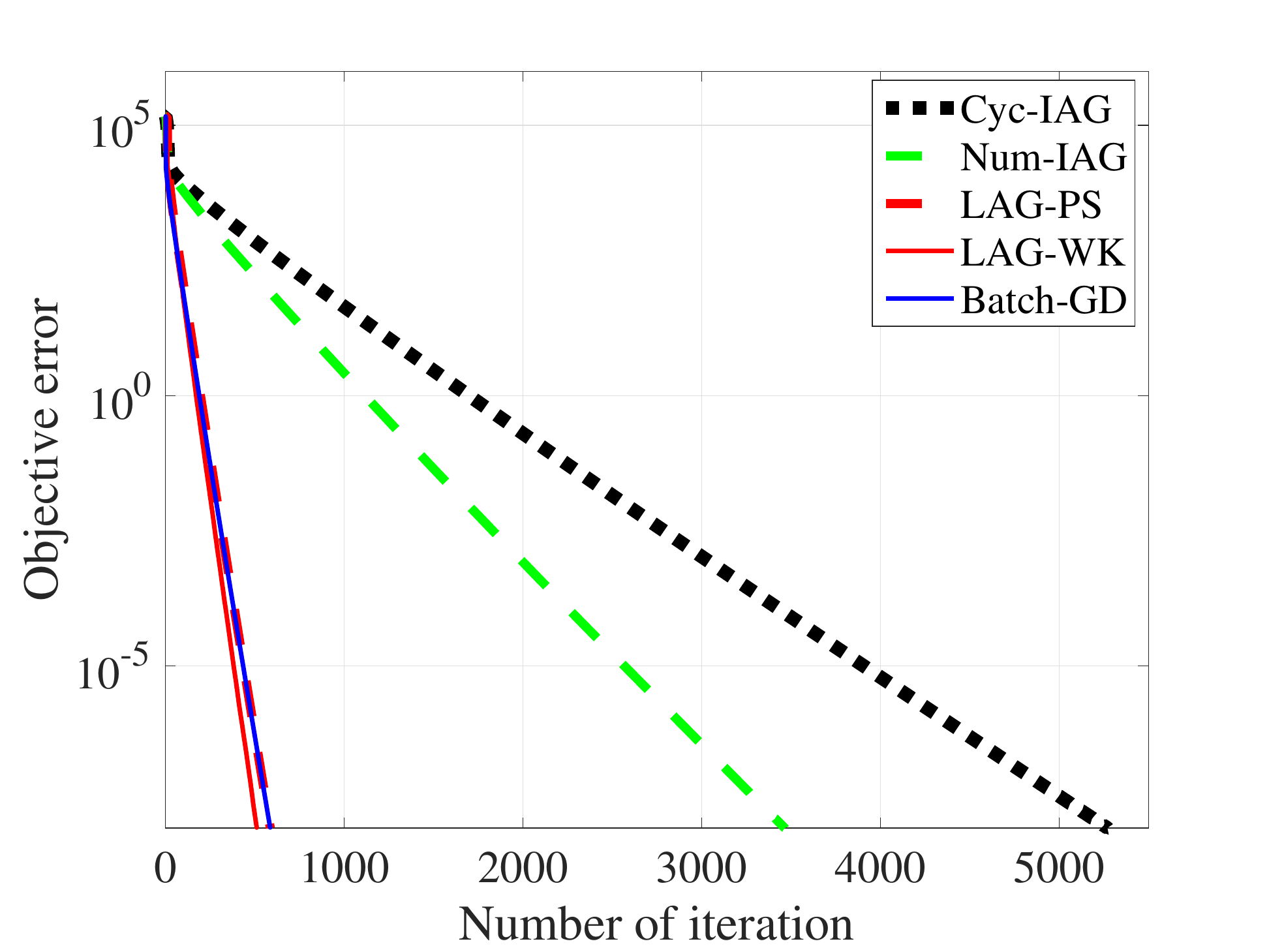}&
\includegraphics[width=6.5cm]{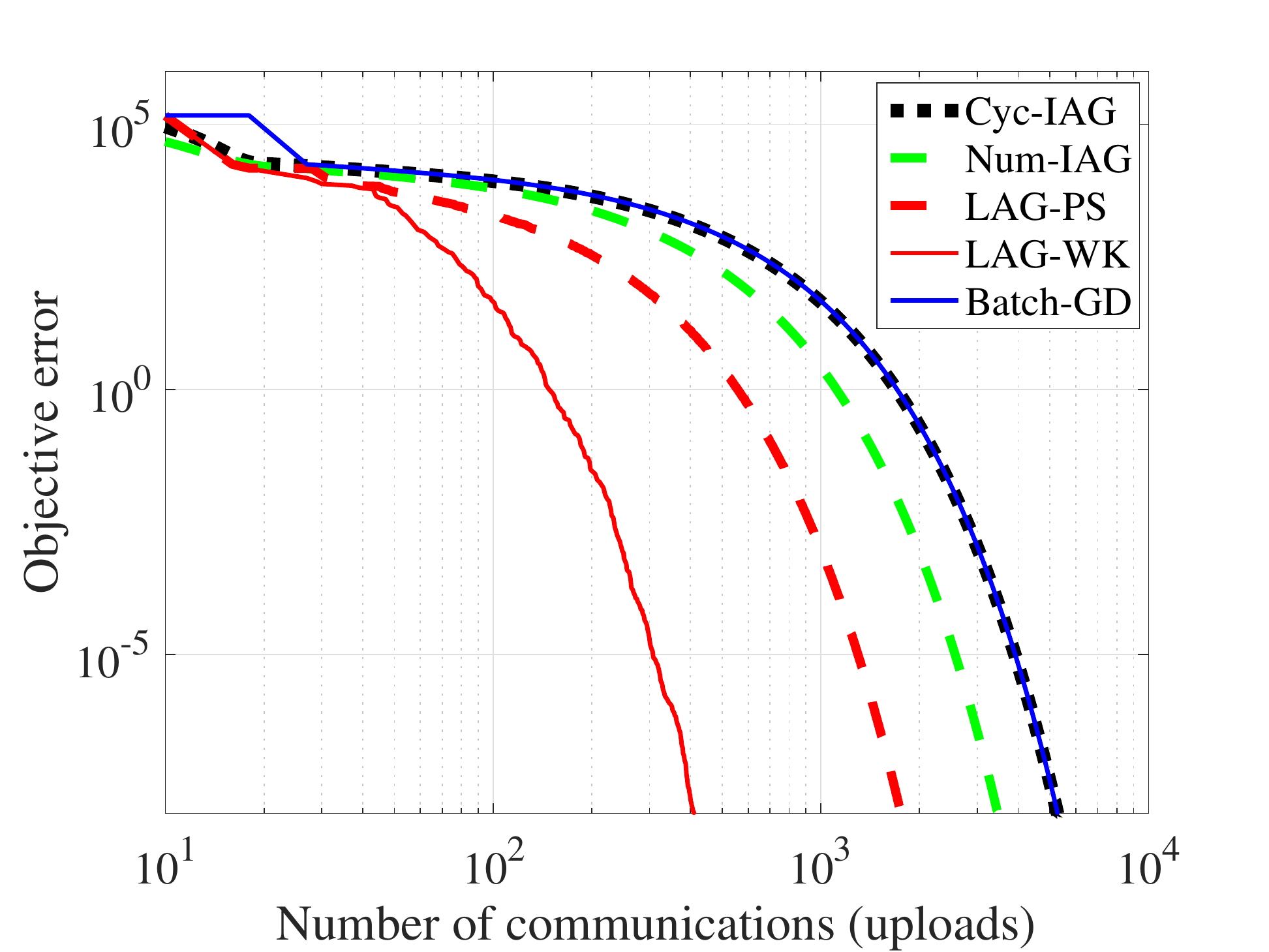}
\end{tabular}
\vspace*{-0.1cm}
  \caption{Iteration and communication complexity for linear regression in real datasets.}
\label{fig:realfig0}
\vspace*{-0.2cm}
\end{figure}

\begin{figure}[t]
\centering
\begin{tabular}{cc}
\includegraphics[width=6.5cm]{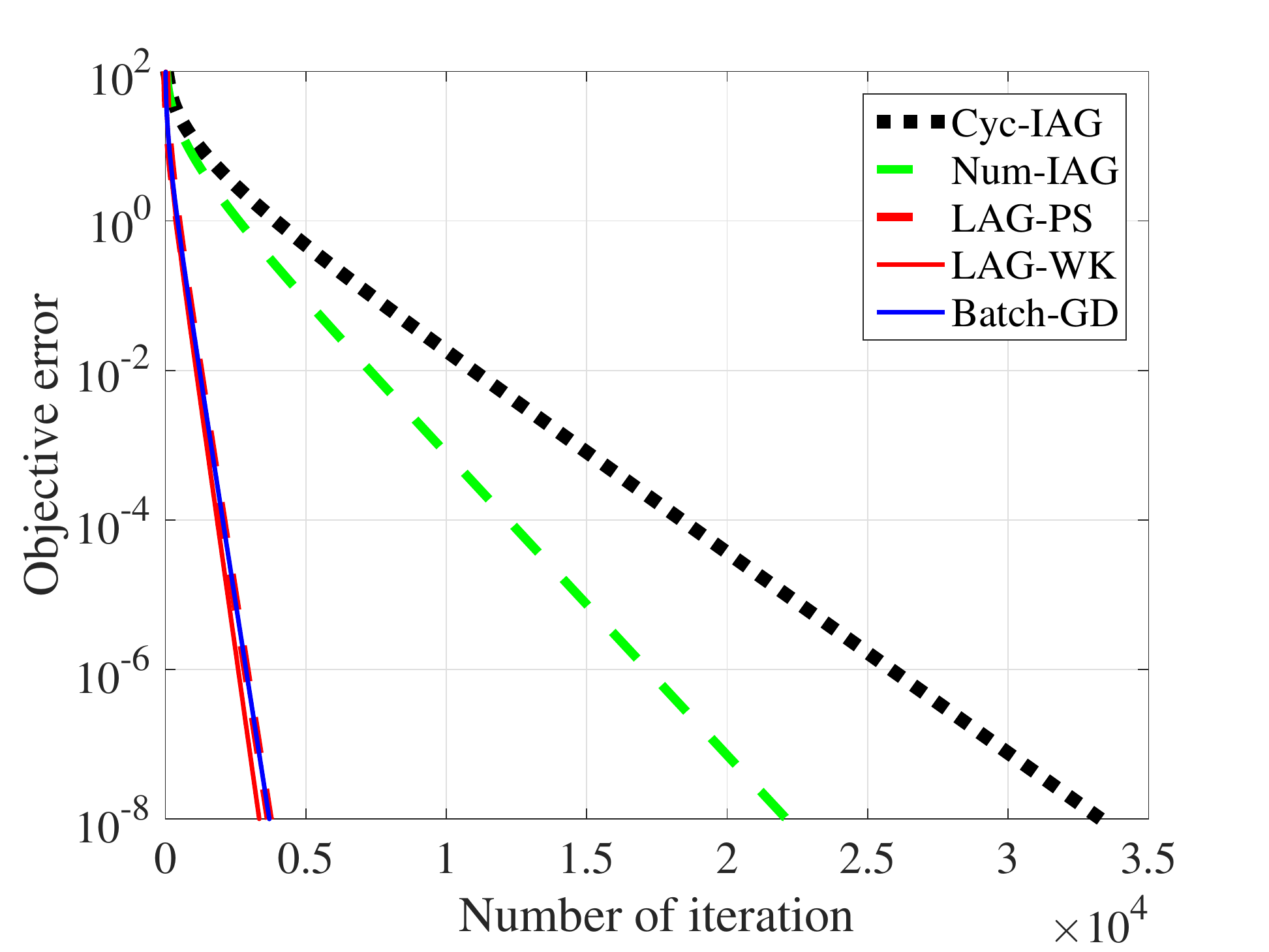}&
\includegraphics[width=6.5cm]{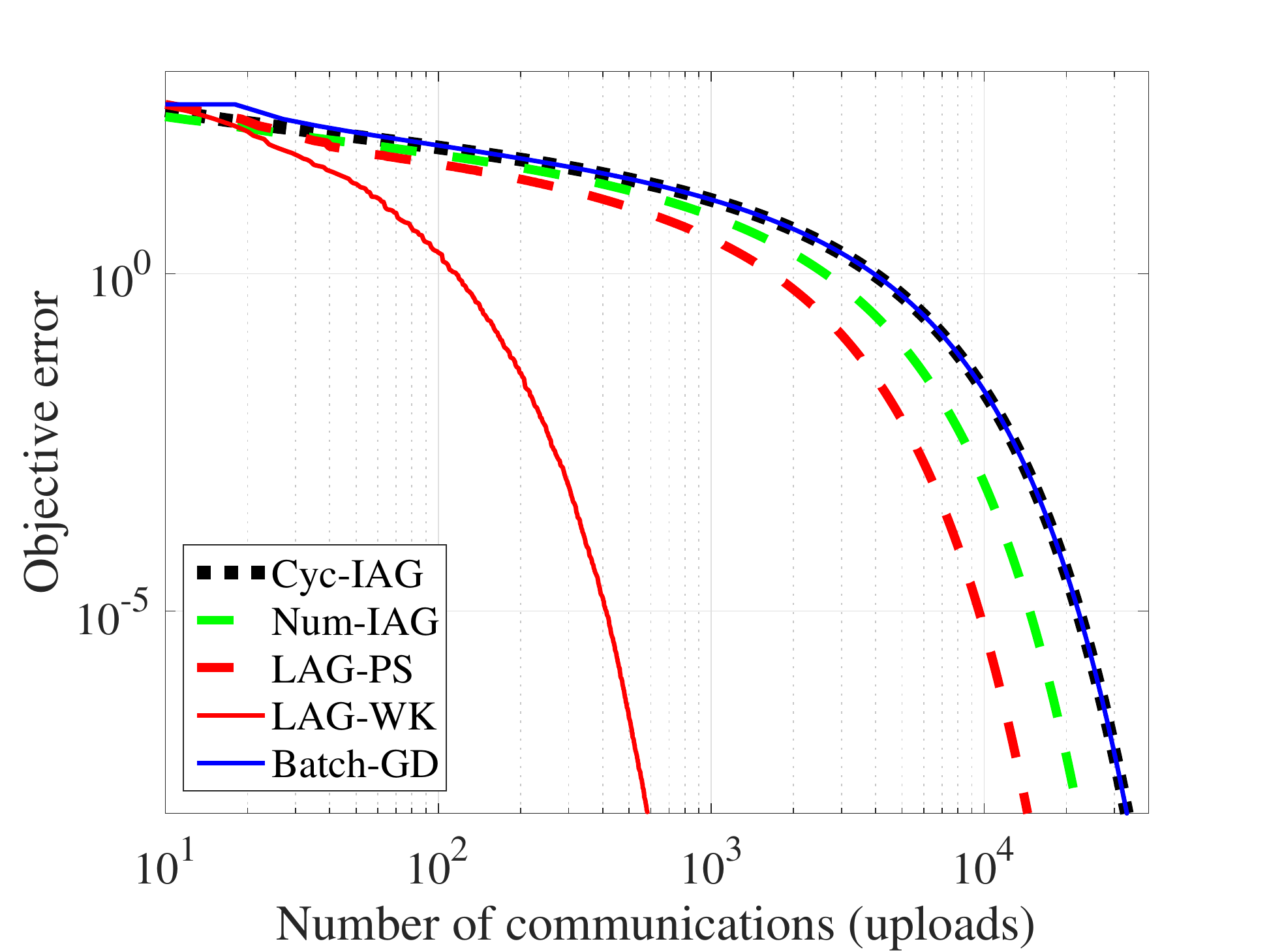}
\end{tabular}
\vspace*{-0.1cm}
  \caption{Iteration and communication complexity for logistic regression in real datasets.}
\label{fig:realfig1}
\vspace*{-0.2cm}
\end{figure}

We consider two \textbf{synthetic data} tests: a) linear regression with increasing smoothness constants, e.g., $L_m=(1.3^{m-1}+1)^2,\,\forall m$; and, b) logistic regression with uniform smoothness constants, e.g., $L_1=\ldots=L_9=4$.
For each worker, we generate 50 samples $\mathbf{x}_n\in \mathbb{R}^{50}$ from the standard Gaussian distribution, and rescale the data to mimic the increasing and uniform smoothness constants.
For the case of increasing $L_m$, it is not surprising that both LAG variants need fewer communication rounds; see Figure \ref{fig:synfig1}.
Interesting enough, for uniform $L_m$, LAG-WK still has marked improvements on communication, thanks to its ability of exploiting the \emph{hidden} smoothness of the loss functions; that is, the local curvature of ${\cal L}_m$ may not be as steep as $L_m$; see Figure \ref{fig:synfig2}.

\begin{table}[h!]
 \small
\begin{center}
\vspace{-0.1cm}
 \begin{tabular}[t]{ |c | c | c |c| }
\hline
     \textbf{Dataset} & \textbf{\# features} ($d$) & \textbf{\# samples} ($N$)& \textbf{worker index} \\ \hline
    Housing & $13$ &$506$  & 1,2,3\\ \hline
    Body fat& $14$  & $252$ &  4,5,6 \\ \hline
    Abalone& $8$ & $417$& 7,8,9 \\ \hline
      \end{tabular}
\end{center}
\vspace{-0.4cm}
\caption{A summary of real datasets used in the linear regression tests.}
\label{tab:1}
\vspace{-0.1cm}
\end{table}

 \begin{table}[h!]
  \small
\begin{center}
\vspace{-0.3cm}
 \begin{tabular}[t]{ |c | c | c |c| }
\hline
     \textbf{Dataset} & \textbf{\# features} ($d$) & \textbf{\# samples} ($N$)& \textbf{worker index} \\ \hline
    Ionosphere   & $34$ &$351$& 1,2,3\\ \hline
    Adult fat& $113$ &$1605$  &  4,5,6 \\ \hline
    Derm & $34$ & $358$& 7,8,9 \\ \hline
      \end{tabular}
\end{center}
\vspace{-0.4cm}
\caption{A summary of real datasets used in the logistic regression tests.}
\label{tab:2}
\vspace{-0.1cm}
\end{table}

Performance is also tested on the \textbf{real datasets} \citep{Lichman2013}: a) linear regression using \textbf{Housing}, \textbf{Body fat}, \textbf{Abalone} datasets; and, b) logistic regression using \textbf{Ionosphere}, \textbf{Adult}, \textbf{Derm} datasets; see Figures \ref{fig:realfig0}-\ref{fig:realfig1}.
Each dataset is evenly split into three workers with the number of features used in the test equal to the minimal number of features among all datasets; see the summaries of datasets in Tables \ref{tab:1} and \ref{tab:2}, while the details are deferred to Appendix \ref{appsec.data}.
In all tests, LAG-WK outperforms the alternatives in terms of both metrics, especially reducing the needed communication rounds by several orders of magnitude.
Its needed communication rounds can be even \emph{smaller} than the number of iterations, if none of workers violate the trigger condition \eqref{eq.trig-cond} at certain iterations.
Additional tests on real datasets under different number of workers are listed in Table \ref{tab:workercomp}.
Under all the tested settings, LAG-WK consistently achieves the lowest communication complexity, which corroborates the effectiveness of LAG when it comes to communication reduction.
\begin{table}
\small
\begin{center}
 \begin{tabular}{ c || c |c | c ||c |c |c }
\hline \hline
&\multicolumn{3}{|c||}{\textbf{Linear regression}}& \multicolumn{3}{|c}{\textbf{Logistic regression}}\\ \hline \textbf{Algorithm}
     &~~~$M=9$~~~&~~~$M=18$~~~& ~~~$M=27$~~~&~~~$M=9$~~~&~~~$M=18$~~~& ~~~$M=27$~~~\\ \hline \hline
    Cyclic-IAG& $5271$ &$10522$  & $15773$ &$ 33300$& $65287$ &$97773$ \\ \hline
    Num-IAG & $3466$ & $5283$ & $5815$  &$22113$& $30540$ &$37262$\\ \hline
    \textbf{LAG-PS} & $\mathbf{1756}$ &$\mathbf{3610}$  &  $\mathbf{5944}$& $\mathbf{14423}$ & $\mathbf{29968}$ &$\mathbf{44598}$\\ \hline
    \textbf{LAG-WK}& $\mathbf{412}$  & $\mathbf{657}$ & $\mathbf{1058}$  &$\mathbf{584}$  & $\mathbf{1098}$ &$\mathbf{1723}$ \\ \hline
    Batch GD & $5283$ &  $10548$ & $15822$ & $ 33309$ & $65322$ &$97821$\\ \hline\hline
    \end{tabular}
\end{center}
\vspace{-0.3cm}
\caption{Communication complexity to achieve accuracy $\epsilon=10^{-8}$ under different number of workers.}
\label{tab:workercomp}
\vspace{-0.1cm}
\end{table}

\begin{figure}[t]
\centering
\begin{tabular}{cc}
\includegraphics[width=6.5cm]{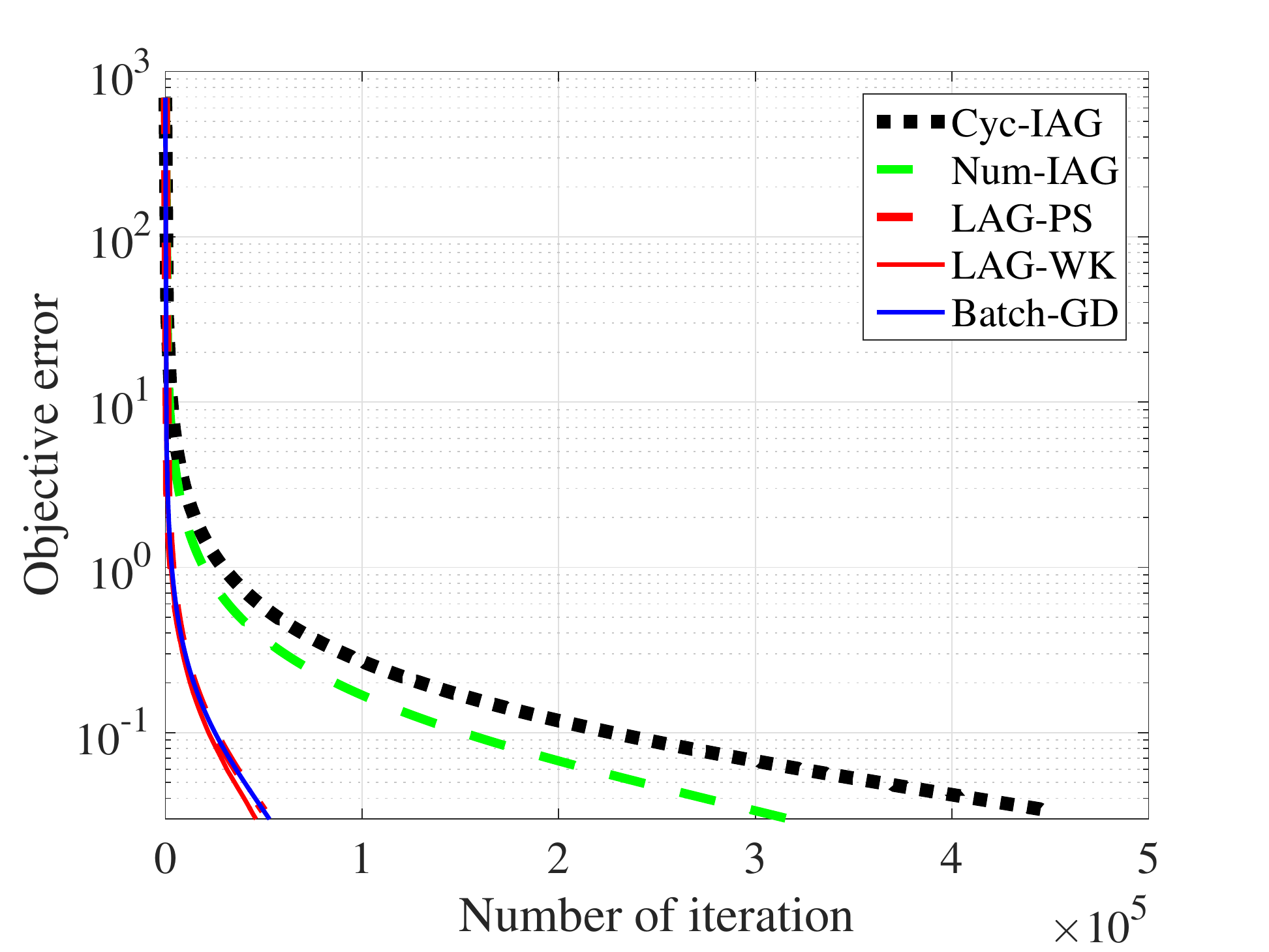}&
\includegraphics[width=6.5cm]{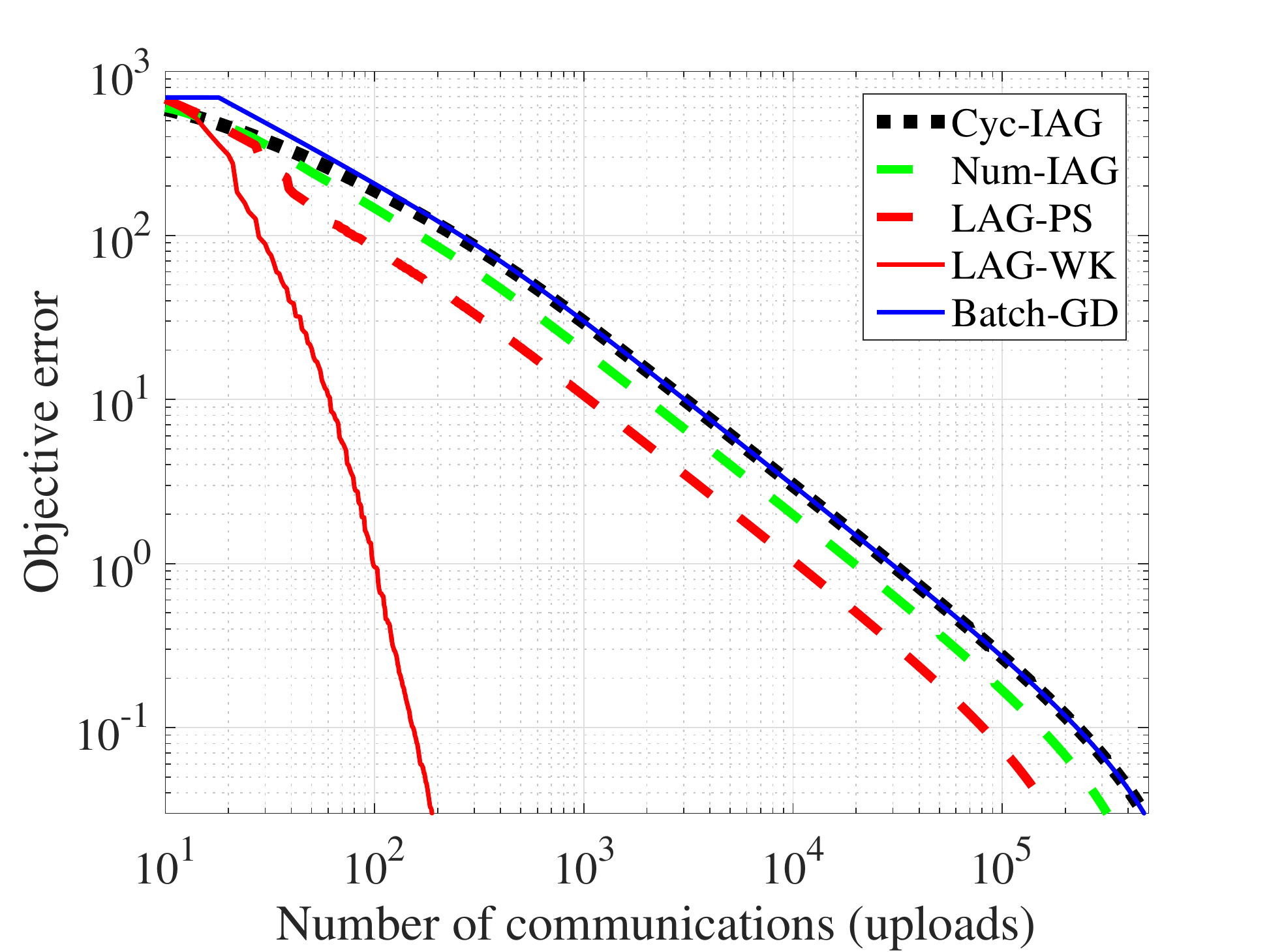}
\end{tabular}
\vspace*{-0.1cm}
  \caption{Iteration and communication complexity in Gisette dataset.}
\label{fig:realfig2}
\vspace*{-0.2cm}
\end{figure}

Similar performance gain has also been observed in the test on a larger dataset \textbf{Gisette}.
The Gisette dataset was constructed from the MNIST data \citep{lecun1998}. After random selecting subset of samples and eliminating all-zero features, it contains $2000$ samples $\mathbf{x}_n\in \mathbb{R}^{4837}$. We randomly split this dataset into nine workers.
The performance of all the algorithms is reported in Figure \ref{fig:realfig2} in terms of the iteration and communication complexity.
Clearly, LAG-WK and LAG-PS achieve the same iteration complexity as GD, and outperform Cyc- and Num-IAG.
Regarding communication complexity, two LAG variants reduce the needed communication rounds by several orders of magnitude compared with the alternatives.

\section{Conclusions}
Confirmed by the impressive empirical performance on both synthetic and real datasets, this paper developed a promising communication-cognizant method for distributed machine learning that we term Lazily Aggregated Gradient (LAG) approach. LAG can achieve the same convergence rates as batch gradient descent (GD) in smooth strongly-convex, convex, and nonconvex cases, and requires fewer communication rounds than GD given that the datasets at different workers are heterogeneous.
To overcome the limitations of LAG, our future work consists of incorporating smoothing techniques to handle nonsmooth loss functions, and robustifying our aggregation rules to deal with cyber attacks.


\clearpage
\appendix
\section{Proof of Lemma \ref{lemma1}}\label{appsec.appdest}
Using the smoothness of ${\cal L}(\cdot)$ in Assumption 1, we have that
\begin{equation}\label{eq.pf-lemma1-1}
	{\cal L}(\bbtheta^{k+1})-{\cal L}(\bbtheta^k)\leq \left\langle \nabla{\cal L}(\bbtheta^k), \bbtheta^{k+1}-\bbtheta^k\right\rangle+\frac{L}{2}\left\|\bbtheta^{k+1}-\bbtheta^k\right\|^2.
\end{equation}

Plugging \eqref{eq.LAG2} into $\left\langle \nabla{\cal L}(\bbtheta^k), \bbtheta^{k+1}-\bbtheta^k\right\rangle$ leads to (cf. $\hat{\bbtheta}_m^k=\hat{\bbtheta}_m^{k-1},\,\forall m\in{\cal M}^k_c$)
\begin{align}\label{eq.pf-lemma1-3}
&\Big\langle \nabla{\cal L}(\bbtheta^k), \bbtheta^{k+1}-\bbtheta^k\Big\rangle\nonumber\\
=&-\alpha\left\langle \nabla{\cal L}(\bbtheta^k),\nabla{\cal L}(\bbtheta^k)+\sum_{m\in{\cal M}^k_c}\left(\nabla{\cal L}_m\big(\hat{\bbtheta}_m^k\big)-\nabla{\cal L}_m\big(\bbtheta^k\big)\right)\right\rangle\nonumber\\
= &-\alpha\left\|\nabla{\cal L}(\bbtheta^k)\right\|^2 -\alpha\left\langle \nabla{\cal L}(\bbtheta^k),\sum_{m\in{\cal M}^k_c}\left(\nabla{\cal L}_m\big(\hat{\bbtheta}_m^k\big)-\nabla{\cal L}_m\big(\bbtheta^k\big)\right)\right\rangle\nonumber\\
=& -\alpha\left\|\nabla{\cal L}(\bbtheta^k)\right\|^2 +\left\langle -\sqrt{\alpha}\nabla{\cal L}(\bbtheta^k),\sqrt{\alpha}\sum_{m\in{\cal M}^k_c}\left(\nabla{\cal L}_m\big(\hat{\bbtheta}_m^k\big)-\nabla{\cal L}_m\big(\bbtheta^k\big)\right)\right\rangle.
\end{align}

Using $2\mathbf{a}^{\top}\mathbf{b}=\|\mathbf{a}\|^2+\|\mathbf{b}\|^2-\|\mathbf{a}-\mathbf{b}\|^2$, we can re-write the inner product in \eqref{eq.pf-lemma1-3} as
\begin{align}\label{eq.pf-lemma1-4}
&\left\langle -\sqrt{\alpha}\nabla{\cal L}(\bbtheta^k),\sqrt{\alpha}\sum_{m\in{\cal M}^k_c}\left(\nabla{\cal L}_m\big(\hat{\bbtheta}_m^k\big)-\nabla{\cal L}_m\big(\bbtheta^k\big)\right)\right\rangle\nonumber\\
=&\frac{\alpha}{2}\left\|\nabla{\cal L}(\bbtheta^k)\right\|^2+\frac{\alpha}{2}\Bigg\|\!\sum_{m\in{\cal M}^k_c}\!\Big(\nabla{\cal L}_m\big(\hat{\bbtheta}_m^k\big)-\nabla{\cal L}_m\big(\bbtheta^k\big)\!\Big)\Bigg\|^2\nonumber\\
&\qquad\qquad\qquad~~-\frac{1}{2}\Bigg\|\sqrt{\alpha}\nabla{\cal L}(\bbtheta^k)+\sqrt{\alpha}\sum_{m\in{\cal M}^k_c}\!\Big(\!\nabla{\cal L}_m\big(\hat{\bbtheta}_m^k\big)-\nabla{\cal L}_m\big(\bbtheta^k\big)\!\Big)\Bigg\|^2\nonumber\\
\stackrel{(a)}{=}&\frac{\alpha}{2}\left\|\nabla{\cal L}(\bbtheta^k)\right\|^2\!+\frac{\alpha}{2}\Bigg\|\!\sum_{m\in{\cal M}^k_c}\!\Big(\nabla{\cal L}_m\big(\hat{\bbtheta}_m^k\big)-\nabla{\cal L}_m\big(\bbtheta^k\big)\!\Big)\Bigg\|^2\!\!-\frac{1}{2\alpha}\left\|\bbtheta^{k+1}-\bbtheta^k\right\|^2
\end{align}
where (a) follows from the LAG update \eqref{eq.LAG2}.

Combining \eqref{eq.pf-lemma1-3} and \eqref{eq.pf-lemma1-4}, and plugging into \eqref{eq.pf-lemma1-1}, the claim of Lemma \ref{lemma1} follows.

\section{Proof of Lemma \ref{lemma2}}\label{appsec.dest}
Using the definition of $\mathbb{V}^k$ in \eqref{eq.Lyap}, it follows that
\begin{align}\label{eq.pf-lemma2-1}
\!\mathbb{V}^{k+1}\!-\mathbb{V}^k=&{\cal L}(\bbtheta^{k+1})-{\cal L}(\bbtheta^k)+\sum_{d=1}^D \beta_d\left\|\bbtheta^{k+2-d}-\bbtheta^{k+1-d}\right\|^2-\sum_{d=1}^D \beta_d\left\|\bbtheta^{k+1-d}-\bbtheta^{k-d}\right\|^2\nonumber\\
\stackrel{(a)}{\leq}&-\frac{\alpha}{2}\left\|\nabla{\cal L}(\bbtheta^k)\right\|^2\!+\!\frac{\alpha}{2}\Bigg\|\!\sum_{m\in{\cal M}^k_c}\!\!\!\Big(\nabla{\cal L}_m\big(\hat{\bbtheta}_m^k\big)\!-\!\nabla{\cal L}_m\big(\bbtheta^k\big)\!\Big)\Bigg\|^2\!\!+\sum_{d=2}^D \beta_d\left\|\bbtheta^{k+2-d}\!-\!\bbtheta^{k+1-d}\right\|^2\nonumber\\
&+\left(\frac{L}{2}-\frac{1}{2\alpha}+\beta_1\right)\left\|\bbtheta^{k+1}-\bbtheta^k\right\|^2-\sum_{d=1}^D \beta_d\left\|\bbtheta^{k+1-d}-\bbtheta^{k-d}\right\|^2
\end{align}
where (a) uses \eqref{eq.lemma1} in Lemma \ref{lemma1}.

Decomposing the square distance as
\begin{align}\label{eq.pf-lemma2-2}
\!\!\!\left\|\bbtheta^{k+1}-\bbtheta^k\right\|^2\!\!= &\Bigg\|\alpha \nabla {\cal L}(\bbtheta^k)+\alpha\sum_{m\in{\cal M}^k_c}\left(\nabla{\cal L}_m\big(\hat{\bbtheta}_m^k\big)-\nabla{\cal L}_m\big(\bbtheta^k\big)\right) \Bigg\|^2\nonumber\\
	\stackrel{(b)}{\leq} &\left(1+\rho\right)\alpha^2\left\|\nabla {\cal L}(\bbtheta^k)\right\|^2\!+\!\left(1+\rho^{-1}\right)\alpha^2\Bigg\|\sum_{m\in{\cal M}^k_c}\!\!\left(\nabla{\cal L}_m\big(\hat{\bbtheta}_m^k\big)\!-\!\nabla{\cal L}_m\big(\bbtheta^k\big)\right)\!\Bigg\|^2\!
\end{align}
where (b) follows from Young's inequality.
Plugging \eqref{eq.pf-lemma2-2} into \eqref{eq.pf-lemma2-1}, we arrive at
\begin{align}\label{eq.pf-lemma2-3}
\!\!\mathbb{V}^{k+1}-\mathbb{V}^k\leq &\left(\left(\frac{L}{2}-\frac{1}{2\alpha}+\beta_1\right)\left(1+\rho\right)\alpha^2-\frac{\alpha}{2}\right)\left\|\nabla {\cal L}(\bbtheta^k)\right\|^2\nonumber\\
+&\sum_{d=1}^{D-1} \left(\beta_{d+1}-\beta_d\right)\left\|\bbtheta^{k+1-d}-\bbtheta^{k-d}\right\|^2-\beta_D\left\|\bbtheta^{k+1-D}-\bbtheta^{k-D}\right\|^2\nonumber\\
+&\left(\left(\frac{L}{2}-\frac{1}{2\alpha}+\beta_1\right)\left(1+\rho^{-1}\right)\alpha^2+\frac{\alpha}{2}\right)\!\Bigg\|\!\sum_{m\in{\cal M}^k_c}\!\!\left(\nabla{\cal L}_m\big(\hat{\bbtheta}_m^k\big)\!-\!\nabla{\cal L}_m\big(\bbtheta^k\big)\right)\!\Bigg\|^2\!.\!
\end{align}

Using $(\sum_{n=1}^N a_n)^2\leq N\sum_{n=1}^N a_n^2$, it follows that
\begin{subequations}\label{eq.pf-lemma2-4}
	\begin{align}
\Bigg\|\sum_{m\in{\cal M}^k_c}\left(\nabla{\cal L}_m\big(\hat{\bbtheta}_m^k\big)-\nabla{\cal L}_m\big(\bbtheta^k\big)\right)\Bigg\|^2\leq &\left|{\cal M}^k_c\right|\sum_{m\in{\cal M}^k_c}\Big\|\nabla{\cal L}_m\big(\hat{\bbtheta}_m^k\big)-\nabla{\cal L}_m\big(\bbtheta^k\big)\Big\|^2\label{eq.pf-lemma2-4-1}\\
\stackrel{(c)}{\leq} &\,\left|{\cal M}^k_c\right|\sum_{m\in{\cal M}^k_c}L_m^2\Big\|\hat{\bbtheta}_m^k-\bbtheta^k\Big\|^2\label{eq.pf-lemma2-4-2}\\
\stackrel{(d)}{\leq}&\frac{|{\cal M}^k_c|^2}{\alpha^2|{\cal M}|^2}\sum_{d=1}^D \xi_d\left\|\bbtheta^{k+1-d}-\bbtheta^{k-d}\right\|^2\label{eq.pf-lemma2-4-3}
\end{align}
\end{subequations}
where (c) follows the smoothness condition in Assumption 1, and (d) uses the trigger condition \eqref{eq.trig-cond1} if we derive from \eqref{eq.pf-lemma2-4-1} to \eqref{eq.pf-lemma2-4-3}, uses \eqref{eq.trig-cond2} if we derive from \eqref{eq.pf-lemma2-4-2} to \eqref{eq.pf-lemma2-4-3}.

Plugging \eqref{eq.pf-lemma2-4} into \eqref{eq.pf-lemma2-3}, we have
\begin{align}\label{eq.pf-lemma2}
&\mathbb{V}^{k+1}-\mathbb{V}^k \nonumber\\
\!\!\leq &\left(\left(\frac{L}{2}-\frac{1}{2\alpha}+\beta_1\right)\left(1+\rho\right)\alpha^2-\frac{\alpha}{2}\right)\left\|\nabla{\cal L}(\bbtheta^k)\right\|^2\nonumber\\
\!\!+&\sum_{d=1}^{D-1}\!\left(\!\!\left(\!\left(\frac{L}{2}-\frac{1}{2\alpha}+\beta_1\right)\!\left(1+\rho^{-1}\right)\alpha^2+\frac{\alpha}{2}\right)\!\frac{\xi_d\left|{\cal M}^k_c\right|^2}{\alpha^2|{\cal M}|^2}-\beta_d+\beta_{d+1}\right)\!\left\|\bbtheta^{k+1-d}\!-\bbtheta^{k-d}\right\|^2\nonumber\\
\!\!+&\left(\!\!\left(\!\left(\frac{L}{2}-\frac{1}{2\alpha}+\beta_1\right)\!\left(1+\rho^{-1}\right)\alpha^2+\frac{\alpha}{2}\right)\!\frac{\xi_D\left|{\cal M}^k_c\right|^2}{\alpha^2|{\cal M}|^2}-\beta_D\right)\left\|\bbtheta^{k+1-D}-\bbtheta^{k-D}\right\|^2.
\end{align}
After defining some constants to simplify the notation, the proof is then complete.

Furthermore, if the stepsize $\alpha$, parameters $\{\beta_d\}$, and the trigger constants $\{\xi_d\}$ satisfy
\begin{subequations}\label{eq.beta}
\begin{align}
\left(\frac{L}{2}-\frac{1}{2\alpha}+\beta_1\right)\left(1+\rho\right)\alpha^2-\frac{\alpha}{2}&\leq 0\label{eq.beta-1}\\
\left(\left(\frac{L}{2}-\frac{1}{2\alpha}+\beta_1\right)\left(1+\rho^{-1}\right)\alpha^2+\frac{\alpha}{2}\right)\frac{\xi_d\left|{\cal M}^k_c\right|^2}{\alpha^2|{\cal M}|^2}-\beta_d+\beta_{d+1} &\leq 0,\,\forall d=1,\ldots,D-1\label{eq.beta-2}\\
\left(\left(\frac{L}{2}-\frac{1}{2\alpha}+\beta_1\right)\!\left(1+\rho^{-1}\right)\alpha^2+\frac{\alpha}{2}\right)\!\frac{\xi_D\left|{\cal M}^k_c\right|^2}{\alpha^2|{\cal M}|^2}-\beta_D&\leq 0\label{eq.beta-3}
\end{align}
\end{subequations}
then Lyapunov function is non-increasing; that is, $\mathbb{V}^{k+1}\leq\mathbb{V}^k$.

\paragraph{Choice of parameters.}
We discuss several choices of parameters that satisfy \eqref{eq.beta}.

\noindent$\bullet$ If $\beta_1=\frac{1-\alpha L}{2\alpha}$ so that $\frac{L}{2}-\frac{1}{2\alpha}+\beta_1=0$, after rearranging terms, \eqref{eq.beta} is equivalent to
\begin{align}\label{eq.beta-cond5}
\alpha \leq \frac{1}{L};~~\xi_d \leq \frac{2\alpha(\beta_d-\beta_{d+1})|{\cal M}|^2}{|{\cal M}^k_c|^2},\,\forall d\in[1,D-1];~~\xi_D\!\leq \frac{2\alpha\beta_D|{\cal M}|^2}{|{\cal M}^k_c|^2}.
\end{align}

\noindent$\bullet$ If $\beta_1\neq \frac{1-\alpha L}{2\alpha}$, after rearranging terms, \eqref{eq.beta} is equivalent to
\begin{subequations}\label{eq.beta-cond}
\begin{align}
\alpha &\leq \frac{1+(1+\rho)^{-1}}{L+2\beta_1};\label{eq.beta-cond-1}\\
\xi_d &\leq \frac{2\alpha(\beta_d-\beta_{d+1})|{\cal M}|^2}{\left((1+\rho^{-1})(2\alpha\beta_1+\alpha L-1)+1\right)|{\cal M}^k_c|^2},~~~ d=1,\ldots,D-1\label{eq.beta-cond-2}\\
\xi_D\!&\leq \frac{2\alpha\beta_D|{\cal M}|^2}{\left((1+\rho^{-1})(2\alpha\beta_1+\alpha L-1)+1\right)|{\cal M}^k_c|^2}.\label{eq.beta-cond-3}
\end{align}
\end{subequations}
i) If $\rho\rightarrow 0$ and $\beta_1\rightarrow 0$, \eqref{eq.beta-cond-1} becomes $0\leq \alpha \leq 2/L$, matching the stepsize region of GD.\\
ii) If $\alpha=1/L$ and $\beta_1>0$, \eqref{eq.beta-cond-2} and \eqref{eq.beta-cond-3} reduce to
\begin{equation}\label{eq.beta-cond4}
\xi_d \leq \frac{2\alpha(\beta_d-\beta_{d+1})|{\cal M}|^2}{(2\alpha\beta_1(1+\rho^{-1})+1)|{\cal M}^k_c|^2}~~~{\rm and}~~~\xi_D\!\leq \frac{2\alpha\beta_D|{\cal M}|^2}{\left(2\alpha\beta_1(1+\rho^{-1})+1\right)|{\cal M}^k_c|^2}.	
\end{equation}

Since \eqref{eq.beta-cond5} is in a simpler form, we will use this choice in the subsequent iteration and communication analysis for brevity.

\section{Proof of Theorem \ref{theorem2}}\label{appsec.scvx}
Using Lemma \ref{lemma2}, it follows that (with $\tilde{c}(\alpha,\beta_1):=\frac{L}{2}-\frac{1}{2\alpha}+\beta_1$)
\begin{align}\label{eq.pf-theorem2-2}
&\mathbb{V}^{k+1}-\mathbb{V}^k \nonumber\\
\!\!\leq &-\left(\frac{\alpha}{2}-\tilde{c}(\alpha,\beta_1)\left(1+\rho\right)\alpha^2\right)\left\|\nabla{\cal L}(\bbtheta^k)\right\|^2\nonumber\\
\!\!-&\left(\beta_D-\left(\tilde{c}(\alpha,\beta_1)\left(1+\rho^{-1}\right)\alpha^2+\frac{\alpha}{2}\right)\!\frac{\xi_D\left|{\cal M}^k_c\right|^2}{\alpha^2|{\cal M}|^2}\right)\left\|\bbtheta^{k+1-D}-\bbtheta^{k-D}\right\|^2\nonumber\\
\!\!-&\sum_{d=1}^{D-1}\!\left(\!\beta_d-\beta_{d+1}-\left(\tilde{c}(\alpha,\beta_1)\left(1+\rho^{-1}\right)\alpha^2+\frac{\alpha}{2}\right)\!\frac{\xi_d\left|{\cal M}^k_c\right|^2}{\alpha^2|{\cal M}|^2}\right)\!\left\|\bbtheta^{k+1-d}\!-\bbtheta^{k-d}\right\|^2\!\nonumber\\
\!\!\stackrel{(a)}{\leq}&\!-\!\!\left(\alpha\mu\!-\!2\tilde{c}(\alpha,\beta_1)\left(1+\rho\right)\mu\alpha^2\right)\!\!\left({\cal L}(\bbtheta^k)\!-\!{\cal L}(\bbtheta^*)\right)\nonumber\\
\!\!-&\left(\beta_D-\left(\tilde{c}(\alpha,\beta_1)\left(1+\rho^{-1}\right)\alpha^2+\frac{\alpha}{2}\right)\!\frac{\xi_D\left|{\cal M}^k_c\right|^2}{\alpha^2|{\cal M}|^2}\right)\left\|\bbtheta^{k+1-D}-\bbtheta^{k-D}\right\|^2\nonumber\\
\!\!-&\sum_{d=1}^{D-1}\!\left(\beta_d-\beta_{d+1}\!-\!\left(\tilde{c}(\alpha,\beta_1)\left(1+\rho^{-1}\right)\alpha^2\!+\frac{\alpha}{2}\right)\!\frac{\xi_d\left|{\cal M}^k_c\right|^2}{\alpha^2|{\cal M}|^2}\right)\!\left\|\bbtheta^{k+1-d}\!-\bbtheta^{k-d}\right\|^2\!\!
\end{align}
where (a) uses the strong convexity or the PL condition in Assumption 3, e.g.,
\begin{equation}\label{eq.scvx}
2\mu\left({\cal L}(\bbtheta^k)-{\cal L}(\bbtheta^*)\right)\leq \left\|\nabla{\cal L}(\bbtheta^k)\right\|^2.
\end{equation}
With the constant $c(\alpha;\{\xi_d\})$ defined as
\begin{align}\label{eq.pf-theorem2-3}
\!\!\!c(\alpha;\{\xi_d\})\!:=\min_k\min_{d=1,\ldots,D-1}  \Bigg\{&\alpha\mu-2\tilde{c}(\alpha,\beta_1)\left(1+\rho\right)\mu\alpha^2, 1\!-\!\left(\tilde{c}(\alpha,\beta_1)\left(1+\rho^{-1}\right)\alpha^2+\frac{\alpha}{2}\right)\frac{\xi_D\left|{\cal M}^k_c\right|^2}{\alpha^2\beta_D|{\cal M}|^2},\nonumber\\
&1-\frac{\beta_{d+1}}{\beta_d}-\left(\tilde{c}(\alpha,\beta_1)\left(1+\rho^{-1}\right)\alpha^2+\frac{\alpha}{2}\right)\frac{\xi_d\left|{\cal M}^k_c\right|^2}{\alpha^2\beta_d|{\cal M}|^2}\Bigg\}
\end{align}
we have from \eqref{eq.pf-theorem2-2} that
\begin{align}\label{eq.pf-theorem2-4}
\mathbb{V}^{k+1}-\mathbb{V}^k\stackrel{(b)}{\leq}&-c(\alpha;\{\xi_d\})\left({\cal L}(\bbtheta^k)-{\cal L}(\bbtheta^*)+\sum_{d=1}^D \beta_d\left\|\bbtheta^{k+1-d}-\bbtheta^{k-d}\right\|^2\right)\nonumber\\
=&-c(\alpha;\{\xi_d\})\mathbb{V}^k.
\end{align}

Rearranging terms in \eqref{eq.pf-theorem2-4}, we can conclude that
\begin{equation}\label{eq.pf-linear-const}
	\mathbb{V}^{k+1}\leq\left(1-c(\alpha;\{\xi_d\})\right)\mathbb{V}^k.
\end{equation}
The $Q$-linear convergence of $\mathbb{V}^k$ implies the $R$-linear convergence of ${\cal L}(\bbtheta^k)\!-\!{\cal L}(\bbtheta^*)$.\\ 
The proof is then complete.

\paragraph{Iteration complexity.}
Since the linear rate constant in \eqref{eq.pf-linear-const} is in a complex form, we discuss the iteration complexity under a set of specific parameters (not necessarily optimal).
Specifically, we choose
\begin{equation}\label{eq.para-set1}
\xi_1=\ldots=\xi_D:=\xi<\frac{1}{D}	~~~~~~{\rm and}~~~~~~\alpha:=\frac{1-D\xi/\eta}{L} ~~~~~~{\rm and}~~~~~~\beta_d:=\frac{(D-d+1)\xi}{2\alpha \eta},\,\forall d=1,\cdots,D
\end{equation}
where $\eta$ is a constant. Clearly, \eqref{eq.para-set1} satisfies the condition in \eqref{eq.beta-cond5}.

Plugging \eqref{eq.para-set1} into \eqref{eq.pf-theorem2-3}, we have (cf. $\tilde{c}(\alpha,\beta_1)=0$)
\begin{equation}\label{eq.para-set2}
\Gamma:=1-c(\alpha;\{\xi_d\})=\max_k\max_{d=1,\ldots,D}\left\{1-\frac{1-D\xi/\eta}{\kappa},\,\frac{\eta \left|{\cal M}^k_c\right|^2}{\left|{\cal M}\right|^2},\, \frac{D-d+\eta \left|{\cal M}^k_c\right|^2/\left|{\cal M}\right|^2}{D-d+1}\right\}.\!\!
\end{equation}
If we choose $\eta:=\sqrt{D\xi}$ such that $\frac{\eta \left|{\cal M}^k_c\right|^2}{\left|{\cal M}\right|^2}<1$, we can simplify \eqref{eq.para-set2} as
\begin{equation}\label{eq.para-set3}
		\Gamma=\max_k\left\{1-\frac{1-\sqrt{D\xi}}{\kappa},\, \frac{D-1+\sqrt{D\xi} \left|{\cal M}^k_c\right|^2/\left|{\cal M}\right|^2}{D}  \right\}\stackrel{(a)}{=}1-\frac{1-\sqrt{D\xi}}{\kappa}.
\end{equation}
where (a) holds since we choose $D\leq \kappa$.
With the linear convergence rate in \eqref{eq.para-set3}, we can derive the iteration complexity as
\begin{align}\label{eq.pf-iter-cplx}
&	\frac{\mathbb{V}^K}{\mathbb{V}^0}\leq\left(1-\frac{1-\sqrt{D\xi}}{\kappa}\right)^K\leq \epsilon\nonumber\\
\Longrightarrow & K \log\left(1-\frac{1-\sqrt{D\xi}}{\kappa}\right)\leq \log\left(\epsilon\right)\nonumber\\
\Longrightarrow &\log\left(\frac{1}{\epsilon}\right)\leq K \log\left(1-\frac{1-\sqrt{D\xi}}{\kappa}\right)^{-1}\stackrel{(b)}{\leq} \frac{K}{\frac{\kappa}{1-\sqrt{D\xi}}-1}\nonumber\\
\Longrightarrow & K\geq \frac{\kappa}{1-\sqrt{D\xi}}\log\left(\epsilon^{-1}\right)
\end{align}
where (b) uses $\log (1+x)\leq x,\, \forall x>-1$.
Therefore, we can conclude that $\mathbb{I}_{\rm LAG}(\epsilon)=\frac{\kappa}{1-\sqrt{D\xi}}\log\left(\epsilon^{-1}\right)$.

\section{Proof of Lemma \ref{lemma4}}\label{appsec.lemma4}
The idea is essentially to show that if \eqref{eq.lemma4} holds, then for any iteration $k$, the worker $m$ will not violate the trigger conditions in \eqref{eq.trig-cond} so that does not communicate with the server at the current iteration, if it has communicated with the server at least once during the previous consecutive $d$ iterations.

Suppose at iteration $k$, the most recent iteration that the worker $m$ did communicate with the server is iteration $k-d'$ with $1\leq d'\leq d$.
Thus, we have $\hat{\bbtheta}_m^{k-1}=\bbtheta^{k-d'}$, which implies that
\begin{align}\label{eq.lemma4-1}
L_m^2\left\|\hat{\bbtheta}_m^{k-1}-\bbtheta^k\right\|^2&=L_m^2\left\|\bbtheta^{k-d'}-\bbtheta^k\right\|^2\nonumber\\
&=d'L^2\mathds{H}^2(m)\sum_{b=1}^{d'}\left\|\bbtheta^{k+1-b}-\bbtheta^{k-b}\right\|^2\nonumber\\
&\stackrel{(a)}{\leq} \frac{\xi_d}{\alpha^2|{\cal M}|^2}\sum_{b=1}^{d'}\left\|\bbtheta^{k+1-b}-\bbtheta^{k-b}\right\|^2\nonumber\\
&\stackrel{(b)}{\leq} \frac{\sum_{b=1}^{d'}\xi_b\left\|\bbtheta^{k+1-b}-\bbtheta^{k-b}\right\|^2}{\alpha^2|{\cal M}|^2}+\frac{\sum_{b=d'+1}^{D}\xi_b\left\|\bbtheta^{k+1-b}-\bbtheta^{k-b}\right\|^2}{\alpha^2|{\cal M}|^2}\nonumber\\
&={\rm RHS~of~\eqref{eq.trig-cond2}}
\end{align}
where (a) follows since the condition \eqref{eq.lemma4} is satisfied, so that
\begin{equation}
	\mathds{H}^2(m)\leq \frac{\xi_d}{d\alpha^2 L^2 M^2}\leq  \frac{\xi_d}{d'\alpha^2 L^2 M^2}
\end{equation}
and (b) follows from our choice of $\{\xi_d\}$ such that for $1\leq d'\leq d$, we have $\xi_d\leq \xi_{d'}\leq \ldots \leq \xi_1$ and $\left\|\bbtheta^{k+1-b}-\bbtheta^{k-b}\right\|^2\geq 0$.
Therefore, the trigger condition \eqref{eq.trig-cond2} does not activate, and the worker $m$ does not communicate with the server at iteration $k$.
With an additional step that $\|\nabla {\cal L}_m(\hat{\bbtheta}_m^{k-1})-\nabla {\cal L}_m(\bbtheta^k)\|^2\leq L_m^2\|\hat{\bbtheta}_m^{k-1}-\bbtheta^k\|^2$, we can also prove that if $\hat{\bbtheta}_m^{k-1}=\bbtheta^{k-d'}$, the trigger condition \eqref{eq.trig-cond1} does not activate either.

Note that the above argument holds for any $1\leq d'\leq d$, and thus if \eqref{eq.lemma4} holds, the worker $m$ communicates with the server at most every other $d$ iterations.

\section{Proof of Proposition \ref{prop.comm}}\label{appsec.propcvx}
The condition of communication reduction given in \eqref{eq.lemma4} is equivalent to
\begin{equation}
	\mathds{H}^2(m)\leq \frac{\xi_d}{\alpha^2 L^2 |{\cal M}|^2 d}:=\gamma_d.
\end{equation}
Together with the definition of heterogeneity score function in \eqref{eq.heter-score}, given $\gamma_d$, the quantity $h\left(\gamma_d\right)$ essentially lower bounds the percentage of workers that communicate with the server at most every other $d$ iterations; that is at most $K/(d+1)$ times until iteration $K$.

To calculate the communication complexity of LAG, we split all the workers into $D+1$ subgroups:\\
${\cal M}_0$ - every worker $m$ that does not satisfy $\mathds{H}^2(m)<\gamma_1$;\\
$\cdots$\\
${\cal M}_d$ - every worker $m$ that does satisfy $\mathds{H}^2(m)<\gamma_d$ but does not satisfy $\mathds{H}^2(m)<\gamma_{d+1}$;\\
$\cdots$\\
${\cal M}_D$ - every worker $m$ that does satisfy $\mathds{H}^2(m)<\gamma_D$.

The above splitting is according to our claims in Lemma \ref{lemma4}, which splits all the workers without overlapping. The neat thing is that for workers in each subgroup ${\cal M}_d$, we can upper bound its communication rounds until the current iteration.
Hence, the total communication complexity of LAG is upper bounded by
\begin{align}\label{eq.comm}
\mathbb{C}_{\rm LAG}(\epsilon)=&\sum_{m\in{\cal M}} {\rm Communication~rounds~of~worker}~m\nonumber\\
= &\sum_{d=0}^D ~{\rm Total~communication~rounds~of~workers~in}~{\cal M}_d\nonumber\\
= &\sum_{d=0}^D~ |{\cal M}_d|~\times ~\frac{\mathbb{I}_{\rm LAG}(\epsilon)}{d+1}\nonumber\\
\stackrel{(a)}{\leq} & 	 \left(1-h\left(\gamma_1\right)\!+\!\frac{1}{2}\Big(h\left(\gamma_1\right)-h\left(\gamma_2\right)\Big)\!+\!\ldots\!+\!\frac{1}{D+1}h\left(\gamma_D\right)\right)M~\mathbb{I}_{\rm LAG}(\epsilon)\nonumber\\
=&\Bigg(1-\underbrace{\sum_{d=1}^D\left(\frac{1}{d}-\frac{1}{d+1}\right)h\left(\gamma_d\right)}_{\Delta\bar{\mathbb{C}}(h;\{\gamma_d\})}\Bigg)M~\mathbb{I}_{\rm LAG}(\epsilon):=\left(1-\Delta\bar{\mathbb{C}}(h;\{\gamma_d\})\right)M~\mathbb{I}_{\rm LAG}(\epsilon)
\end{align}
where (a) uses the definition of subgroups $\{{\cal M}_d\}$ and the function $h(\gamma)$ in \eqref{eq.heter-score}.

\begin{figure}[t]
\vspace{-0.1cm}
\centering
\includegraphics[width=8cm]{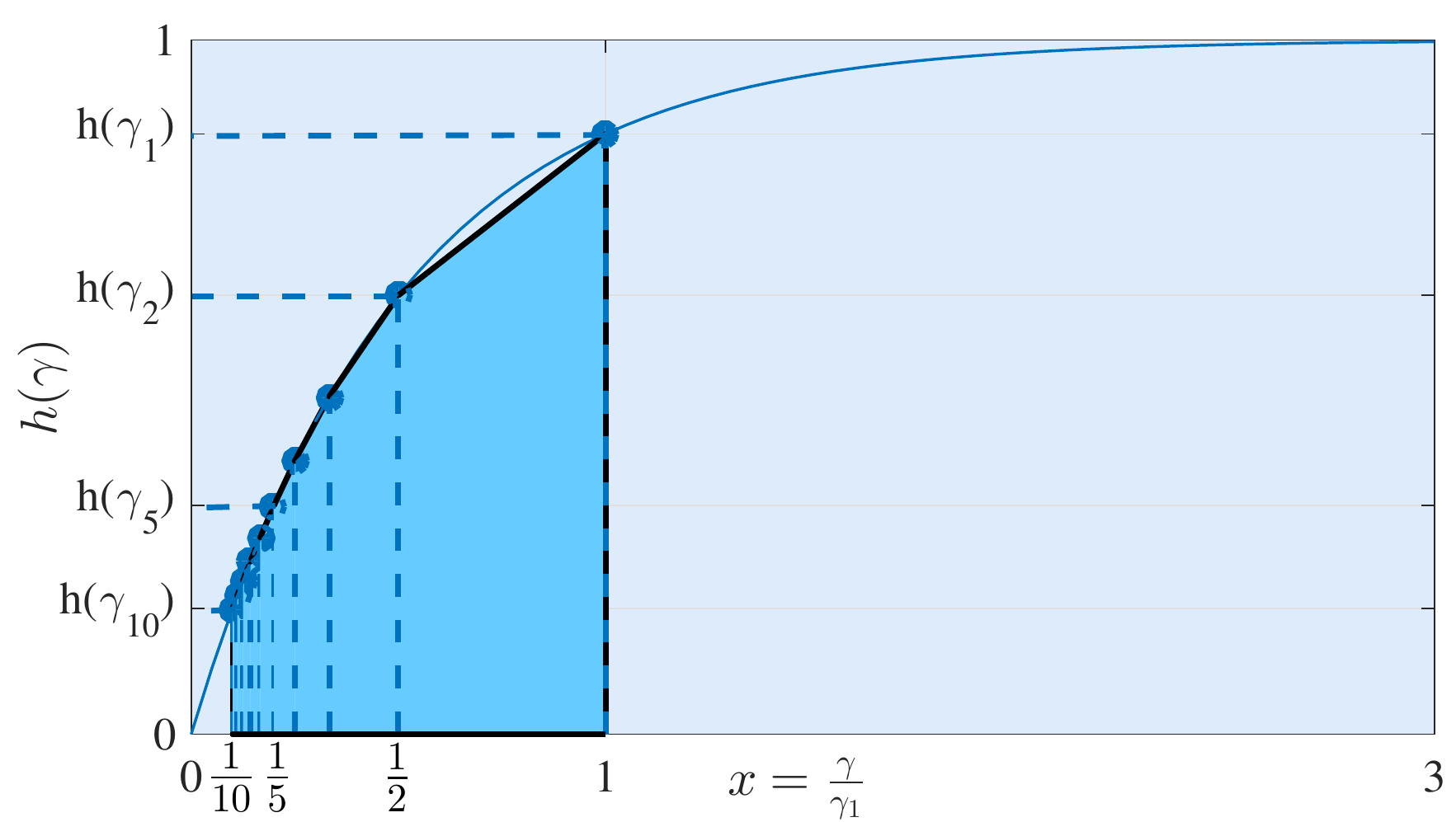}
\vspace*{-4pt}
  \caption{The area of the light blue polygon lower bounds the quantity $\Delta\bar{\mathbb{C}}(h;\xi)$ in \eqref{pf-eq.comm2-2}. It is generated according to $\gamma_d\!:={1}/{(d \gamma_1)}$ and $D\!=\!10$.}
\label{fig:hcdf}
\vspace{-0.2cm}
\end{figure}

If we choose the parameters as those in \eqref{eq.para-set1}, we can simplify the expression of \eqref{eq.comm} and arrive at
\begin{equation}\label{eq.comm2}
	\mathbb{C}_{\rm LAG}(\epsilon)\leq \left(1-\Delta\bar{\mathbb{C}}(h;\xi)\right)\frac{M\kappa}{1-\sqrt{D\xi}} \log (\epsilon^{-1})
\end{equation}
where $\Delta\bar{\mathbb{C}}(h;\{\gamma_d\})$ is written as $\Delta\bar{\mathbb{C}}(h;\xi)$  in this case, because $\gamma_d:=\frac{\xi}{(1-\sqrt{D\xi})^2 M^2 d},\,\forall d$.

On the other hand, even with a larger stepsize $\alpha=1/L$, the communication complexity of GD is $\mathbb{C}_{\rm GD}(\epsilon):=M \kappa \log (\epsilon^{-1})$.
Therefore, if we can show that
\begin{equation}
\frac{1-\Delta\bar{\mathbb{C}}(h;\xi)}{1-\sqrt{D\xi}}\leq 1~~~\iff~~~\sqrt{D\xi}\leq \Delta\bar{\mathbb{C}}(h;\xi)
\end{equation}
then it is safe to conclude that the communication complexity of LAG is lower than that of GD.
Using the nondecreasing property of $h$, we have that (cf. the area of the light blue polygon in Figure \ref{fig:hcdf})
\begin{equation}\label{pf-eq.comm2-2}
\!	\Delta\bar{\mathbb{C}}(h;\xi)\!\in\!\left[\frac{Dh(\gamma_D)}{D+1},\frac{Dh(\gamma_1)}{D+1}\right]\subseteq \left[0,\frac{D}{D+1}\right]
\end{equation}
where we use the fact that $0\leq h(\gamma)\leq 1$.
Since for any $\xi\in(0,1/D)$, there exists a function $h$ such that $\Delta\bar{\mathbb{C}}(h;\xi)$ achieves any value within $\left[0,D/(D+1)\right]$.
Therefore, we can conclude that if $\xi\leq \frac{D}{(D+1)^2}$ so that $\sqrt{D\xi}\leq D/(D+1)$, there always exists $h(\gamma)$ or a distributed learning setting such that $\mathbb{C}_{\rm LAG}(\epsilon)<\mathbb{C}_{\rm GD}(\epsilon)$.


\section{Proof of Theorem \ref{theorem1}}\label{appsec.cvx}
Before establishing the convergence in the convex case, we present a critical lemma.

 \begin{lemma}
\label{lemma3}
	Under Assumptions 1-2, the sequences of Lyapunov functions $\{\mathbb{V}^k\}$ satisfy
\begin{align}\label{eq.lemma3}
\!\!\!\!\!\!\big(\mathbb{V}^k\big)^2\!\!&\leq\!  \left(\left\|\nabla {\cal L}(\bbtheta^k)\right\|^2\!+\!\sum_{d=1}^D \beta_d\left\|\bbtheta^{k+1-d}-\bbtheta^{k-d}\right\|^2\right)\!\left(\left\|\bbtheta^k\!-\!\bbtheta^*\right\|^2\!+\!\sum_{d=1}^D \beta_d\left\|\bbtheta^{k+1-d}-\bbtheta^{k-d}\right\|^2\right)\nonumber\\
&:=\qquad\qquad\qquad\quad\overline{\mathbb{V}}^k(1)\qquad\qquad\qquad~~\times\qquad\qquad\qquad\overline{\mathbb{V}}^k(2)
	\end{align}
where $\overline{\mathbb{V}}^k(1)$ and $\overline{\mathbb{V}}^k(2)$ denote the two terms upper bounding $\left(\mathbb{V}^k\right)^2$, respectively.
\end{lemma}
\begin{proof}
Define two vectors as
\begin{subequations}
	\begin{align}
\mathbf{a}^k:=&\left[\nabla^{\top}{\cal L}(\bbtheta^k),\sqrt{\beta_1}\left\|\bbtheta^k-\bbtheta^{k-1}\right\|,\ldots,\sqrt{\beta_D}\left\|\bbtheta^{k+1-D}-\bbtheta^{k-D}\right\|\right]^{\top}\\	 
\mathbf{b}^k:=&\left[(\bbtheta^k-\bbtheta^*)^{\top},\sqrt{\beta_1}\left\|\bbtheta^k-\bbtheta^{k-1}\right\|,\ldots,\sqrt{\beta_D}\left\|\bbtheta^{k+1-D}-\bbtheta^{k-D}\right\|\right]^{\top}.	 
\end{align}
\end{subequations}

The convexity of ${\cal L}(\bbtheta)$ implies that
	\begin{equation}\label{eq.pf-lemma3-1}
	{\cal L}(\bbtheta^k)-{\cal L}(\bbtheta^*)\leq \langle \nabla{\cal L}(\bbtheta^k), \bbtheta^k-\bbtheta^*\rangle.	
	\end{equation}
Recalling the definition of $\mathbb{V}^k$ in \eqref{eq.Lyap}, it follows that
\begin{align}\label{eq.pf-lemma3-2}
	\mathbb{V}^k&={\cal L}(\bbtheta^k)-{\cal L}(\bbtheta^*)+\sum_{d=1}^D \beta_d\left\|\bbtheta^{k+1-d}-\bbtheta^{k-d}\right\|^2\nonumber\\
&\leq \langle \mathbf{a}^k, \mathbf{b}^k \rangle\leq \|\mathbf{a}^k\|\|\mathbf{b}^k\|
\end{align}
and squaring both sides of \eqref{eq.pf-lemma3-2} leads to
\begin{align}\label{eq.pf-lemma3-3}
\!\left(\mathbb{V}^k\right)^2\!\!\leq\! \left(\left\|\nabla {\cal L}(\bbtheta^k)\right\|^2\!+\sum_{d=1}^D \beta_d\left\|\bbtheta^{k+1-d}-\bbtheta^{k-d}\right\|^2\right)\!\left(\left\|\bbtheta^k\!-\!\bbtheta^*\right\|^2\!+\sum_{d=1}^D \beta_d\left\|\bbtheta^{k+1-d}-\bbtheta^{k-d}\right\|^2\right)
	\end{align}
from which we can conclude the proof.
\end{proof}

Now we are ready to prove Theorem \ref{theorem1}.
Lemma \ref{lemma2} implies that
\begin{align}\label{eq.pf-theorem1-2}
\mathbb{V}^{k+1}-\mathbb{V}^k\leq &-\left(\frac{\alpha}{2}-\tilde{c}(\alpha,\beta_1)\left(1+\rho\right)\alpha^2\right)\!\left\|\nabla{\cal L}(\bbtheta^k)\right\|^2\nonumber\\
\!\!-&\left(\beta_D-\left(\!\tilde{c}(\alpha,\beta_1)\!\left(1+\rho^{-1}\right)\alpha^2+\frac{\alpha}{2}\right)\!\frac{\xi_D\left|{\cal M}^k_c\right|^2}{\alpha^2|{\cal M}|^2}\right)\left\|\bbtheta^{k+1-D}-\bbtheta^{k-D}\right\|^2\nonumber\\
\!\!-&\sum_{d=1}^{D-1}\!\left(\!\beta_d-\beta_{d+1}-\left(\!\tilde{c}(\alpha,\beta_1)\!\left(1+\rho^{-1}\right)\alpha^2+\frac{\alpha}{2}\right)\!\frac{\xi_d\left|{\cal M}^k_c\right|^2}{\alpha^2|{\cal M}|^2}\right)\!\left\|\bbtheta^{k+1-d}\!-\bbtheta^{k-d}\right\|^2\!\nonumber\\
\leq &-c(\alpha;\{\xi_d\})\left(\left\|\nabla {\cal L}(\bbtheta^k)\right\|^2\!+\!\sum_{d=1}^D \beta_d\left\|\bbtheta^{k+1-d}-\bbtheta^{k-d}\right\|^2\right)\nonumber\\
=&-c(\alpha;\{\xi_d\})\overline{\mathbb{V}}^k(1)
\end{align}
where the definition of $c(\alpha;\{\xi_d\})$ is given by
\begin{align}\label{eq.pf-theorem1-c}
c(\alpha;\{\xi_d\}):=\min_{k}\Bigg\{\frac{\alpha}{2}\!-\!\tilde{c}(\alpha,\beta_1)\left(1+\rho\right)&\alpha^2,1\!-\!\left(\tilde{c}(\alpha,\beta_1)\left(1+\rho^{-1}\right)\alpha^2+\frac{\alpha}{2}\right)\frac{\xi_D\left|{\cal M}^k_c\right|^2}{\alpha^2\beta_D|{\cal M}|^2},\nonumber\\
&~1-\frac{\beta_{d+1}}{\beta_d}-\left(\tilde{c}(\alpha,\beta_1)\left(1+\rho^{-1}\right)\alpha^2+\frac{\alpha}{2}\right)\frac{\xi\left|{\cal M}^k_c\right|^2}{\alpha^2\beta_d|{\cal M}|^2}\Bigg\}.
\end{align}

On the other hand, without strong convexity, we can bound $\overline{\mathbb{V}}^k(2)$ as
\begin{equation}\label{eq.pf-theorem1-3}
\overline{\mathbb{V}}^k(2):=\left\|\bbtheta^k\!-\!\bbtheta^*\right\|^2\!+\!\sum_{d=1}^D \beta_d\left\|\bbtheta^{k+1-d}-\bbtheta^{k-d}\right\|^2\leq R	
\end{equation}
where the constant $R$ in the last inequality exists since ${\cal L}(\bbtheta)$ is coercive in Assumption 2 so that ${\cal L}(\bbtheta^*)\leq {\cal L}(\bbtheta^k)<\infty$ implies $\|\bbtheta^k\|< \infty$ thus $\left\|\bbtheta^k\!-\!\bbtheta^*\right\|<\infty$ and $\left\|\bbtheta^k\!-\!\bbtheta^{k-1}\right\|<\infty$.

Plugging \eqref{eq.pf-theorem1-2} and \eqref{eq.pf-lemma1-3} into \eqref{eq.lemma3} in Lemma \ref{lemma3}, we have
\begin{equation}\label{eq.pf-theorem1-4}
	\left(\mathbb{V}^k\right)^2\leq \overline{\mathbb{V}}^k(1)\overline{\mathbb{V}}^k(2)\leq \frac{R}{c(\alpha;\{\xi_d\})}(\mathbb{V}^k-\mathbb{V}^{k+1}).
\end{equation}

Using the fact that the non-increasing property of $\mathbb{V}^k$ in Lemma \ref{lemma2}, we have that
\begin{equation}\label{eq.pf-theorem1-5}
\mathbb{V}^{k+1}\mathbb{V}^k\leq \left(\mathbb{V}^k\right)^2\leq \frac{R}{c(\alpha;\{\xi_d\})}(\mathbb{V}^k-\mathbb{V}^{k+1}).
\end{equation}
Dividing $\mathbb{V}^{k+1}\mathbb{V}^k$ on both sides of \eqref{eq.pf-theorem1-5} and rearranging terms, we have
\begin{equation}\label{eq.pf-theorem1-6}
\frac{c(\alpha;\{\xi_d\})}{R} \leq \frac{1}{\mathbb{V}^{k+1}}-\frac{1}{\mathbb{V}^k}.
\end{equation}
Summing up \eqref{eq.pf-theorem1-6}, it follows that
\begin{equation}
	\frac{K c(\alpha;\{\xi_d\})}{R} \leq \frac{1}{\mathbb{V}^K}-\frac{1}{\mathbb{V}^0}\leq  \frac{1}{\mathbb{V}^{K}}
\end{equation}
from which we can conclude the proof.

\section{Proof of Theorem \ref{theorem0}}\label{appsec.ncvx}
Lemma \ref{lemma2} implies that
\begin{align}\label{eq.pf-theorem0-0}
\mathbb{V}^{k+1}-\mathbb{V}^k\leq &-\left(\frac{\alpha}{2}-\tilde{c}(\alpha,\beta_1)\left(1+\rho\right)\alpha^2\right)\!\left\|\nabla{\cal L}(\bbtheta^k)\right\|^2\nonumber\\
\!\!-&\left(\beta_D-\left(\!\tilde{c}(\alpha,\beta_1)\!\left(1+\rho^{-1}\right)\alpha^2+\frac{\alpha}{2}\right)\!\frac{\xi_D\left|{\cal M}^k_c\right|^2}{\alpha^2|{\cal M}|^2}\right)\left\|\bbtheta^{k+1-D}-\bbtheta^{k-D}\right\|^2\nonumber\\
\!\!-&\sum_{d=1}^{D-1}\!\left(\!\beta_d-\beta_{d+1}-\left(\!\tilde{c}(\alpha,\beta_1)\!\left(1+\rho^{-1}\right)\alpha^2+\frac{\alpha}{2}\right)\!\frac{\xi_d\left|{\cal M}^k_c\right|^2}{\alpha^2|{\cal M}|^2}\right)\!\left\|\bbtheta^{k+1-d}\!-\bbtheta^{k-d}\right\|^2\!\nonumber\\
\leq &-c(\alpha;\{\xi_d\})\left(\left\|\nabla {\cal L}(\bbtheta^k)\right\|^2\!+\!\sum_{d=1}^D \beta_d\left\|\bbtheta^{k+1-d}-\bbtheta^{k-d}\right\|^2\right)
\end{align}

Summing up both sides of \eqref{eq.pf-theorem0-0}, we have
\begin{align}\label{eq.theorem0-1}
c(\alpha;\{\xi_d\})\sum_{k=1}^K\left(\left\|\nabla {\cal L}(\bbtheta^k)\right\|^2\!+\!\sum_{d=1}^D \beta_d\left\|\bbtheta^{k+1-d}-\bbtheta^{k-d}\right\|^2\right)\leq \mathbb{V}^1-\mathbb{V}^{K+1}.
\end{align}
Taking $K\rightarrow \infty$, we have that
\begin{align}\label{eq.theorem0-2}
c(\alpha;\{\xi_d\})\lim_{K\rightarrow \infty}\sum_{k=1}^K\left(\left\|\nabla {\cal L}(\bbtheta^k)\right\|^2\!+\!\sum_{d=1}^D \beta_d\left\|\bbtheta^{k+1-d}-\bbtheta^{k-d}\right\|^2\right)\leq \mathbb{V}^1
\end{align}
where the last inequality holds since the Lyapunov function \eqref{eq.Lyap} is lower bounded by $ \mathbb{V}^k\geq 0,\,\forall k$, and $\mathbb{V}^1<\infty$.
Given the choice of $\alpha$ and $\{\xi_d\}$ in \eqref{eq.beta}, the constant in \eqref{eq.theorem0-2} is $c(\alpha;\{\xi_d\})>0$, and thus two terms in the LHS of \eqref{eq.theorem0-2} are summable, which implies that
\begin{equation}
	\sum_{k=1}^{\infty}\left\|\bbtheta^{k+1}-\bbtheta^k\right\|^2<\infty
\end{equation}	
and likewise that
\begin{equation}
\sum_{k=1}^{\infty}\left\|\nabla{\cal L}(\bbtheta^k)\right\|^2<\infty.
\end{equation}	
Using the implications of summable sequences in \cite[Lemma 3]{davis2016}, the theorem follows.

\section{Proof of Proposition \ref{propncvx}}\label{appsec.propncvx}
Choosing $\beta_d:=\frac{1}{2\alpha}\sum_{\tau=d}^D\xi_{\tau}$ in the Lyapunov function \eqref{eq.Lyap}, we have
\begin{align}
\mathbb{V}^k:={\cal L}(\bbtheta^k)-{\cal L}(\bbtheta^*)+\sum_{d=1}^{D}\frac{(\sum_{j=d}^D\xi_j)}{2\alpha}\|\bbtheta^{k+1-d}-\bbtheta^{k-d}\|^2
\end{align}
Using Lemma \ref{lemma1}, we arrive at
\begin{align}
\mathbb{V}^{k+1}-\mathbb{V}^k\leq -\frac{\alpha}{2}\left\|\nabla{\cal L}(\bbtheta^k)\right\|^2\!+\left(\frac{L}{2}-\frac{1}{2\alpha}+\frac{\sum_{d=1}^D \xi_d}{2\alpha}\right)\left\|\bbtheta^{k+1}-\bbtheta^k\right\|^2.
\end{align}
If the stepsize is chosen as $\alpha=\frac{1}{L}(1-\sum_{d=1}^D\xi_d)$, we have
\begin{equation}
\mathbb{V}^{k+1}-\mathbb{V}^k\leq -\frac{\alpha}{2}\left\|\nabla{\cal L}(\bbtheta^k)\right\|^2.
\end{equation}
Summing up both sides from $k=1,\ldots,K$, and initializing $\bbtheta^{1-D}=\cdots=\bbtheta^0=\bbtheta^1$, we have
\begin{equation}
	\sum_{k=1}^K \left\|\nabla{\cal L}(\bbtheta^k)\right\|^2\leq \frac{2}{\alpha} \mathbb{V}^1=\frac{2}{\alpha} ({\cal L}(\bbtheta^1)-{\cal L}(\bbtheta^*))=\frac{2 L}{1-\sum_{d=1}^D\xi_d} ({\cal L}(\bbtheta^1)-{\cal L}(\bbtheta^*))
\end{equation}
which implies that
\begin{equation}
\min_{k=1,\cdots,K}	\left\|\nabla{\cal L}(\bbtheta^k)\right\|^2\leq \frac{2 L}{(1-\sum_{d=1}^D\xi_d)K} ({\cal L}(\bbtheta^1)-{\cal L}(\bbtheta^*))
\end{equation}
 With regard to GD, it has the following guarantees \citep{nesterov2013}
\begin{equation}
\min_{k=1,\cdots,K}\left\|\nabla{\cal L}(\bbtheta^k)\right\|^2\leq \frac{2 L}{K} ({\cal L}(\bbtheta^1)-{\cal L}(\bbtheta^*)).
\end{equation}
Thus, to achieve the same $\epsilon$-gradient error, the iteration of LAG is $(1-\sum_{d=1}^D\xi_d)^{-1}$ times than GD.
Similar to the derivations in \eqref{eq.comm}, since the LAG's average communication rounds per iteration is $(1-\Delta\bar{\mathbb{C}}(h;\{\gamma_d\}))$ times that of GD, we arrive at \eqref{prononconvex-r1}.

If we choose $\xi_1=\xi_2=\ldots=\xi_D=\xi$, then $\alpha=\frac{1-D\xi}{L}$, and $\gamma_d=\frac{\xi/d}{\alpha^2 L^2M^2}$, $d=1,\ldots,D$.
As $h(\cdot)$ is non-decreasing, if $\gamma_D\geq \gamma'$, we have $h(\gamma_D)\geq h(\gamma')$.
With the definition of $\Delta\bar{\mathbb{C}}(h;\{\gamma_d\})$ in \eqref{eq.prop5}, we get
\begin{align}
\Delta\bar{\mathbb{C}}(h;\{\gamma_d\})=\sum_{d=1}^D\left(\frac{1}{d}-\frac{1}{d+1}\right)h\left(\gamma_d\right)\geq\sum_{d=1}^D\left(\frac{1}{d}-\frac{1}{d+1}\right)h\left(\gamma_D\right)\geq\frac{D}{D+1}h(\gamma').
\end{align}
Therefore, the total communications are reduced if
\begin{align}
\Big(1-\frac{D}{D+1}h(\gamma')\Big)\!\cdot\frac{1}{1- D\xi}<1
\end{align}
which is equivalent to $h(\gamma')>(D+1)\xi$.
The condition $\gamma_D\geq \gamma'$ requires
\begin{align}\label{dayu}
\xi/D\geq \gamma'(1-D\xi)^2|{\cal M}|^2.
\end{align}
Obviously, if $\xi>\gamma' D|{\cal M}|^2$, then \eqref{dayu} holds. In summary, we need
\begin{align}\label{xiao}
\gamma'<\frac{\xi}{D M^2}<\frac{h(\gamma')}{(D+1)D M^2}.
\end{align}
Therefore, we need the function $h$ to satisfy that there exists $\gamma'$ such that \eqref{xiao} holds.

\section{Simulation Details}\label{appsec.data}
This section evaluates the performance of our two LAG algorithms in linear regression tasks with square losses and logistic regression tasks using both synthetic and real-world datasets.

\subsection{Details for linear regression}
For linear regression task, consider the square loss function at worker $m$ as
\begin{align}\label{sec6-obj1}
  {\cal L}_m(\bbtheta):=\sum_{n\in{\cal N}_m}\left(y_n-\mathbf{x}_n^{\top}\bbtheta\right)^2
\end{align}
where $\{\mathbf{x}_n, y_n,\,\forall n\in{\cal N}_m\}$ are data at worker $m$.


\noindent\textbf{Real datasets.} Performance is tested on the following benchmark datasets \citep{Lichman2013}; see a summary in Table \ref{tab:1}.

\noindent$\bullet$ \textbf{Housing} dataset \citep{harrison1978} contains $506$ samples $(\mathbf{x}_n, y_n)$ with $y_n$ representing the median value of house price, which is affected by features in $\mathbf{x}_n$ such as per capita crime rate and weighted distances to five Boston employment centers.

\noindent$\bullet$ \textbf{Body fat} dataset contains $252$ samples $(\mathbf{x}_n, y_n)$ with $y_n$ describing the percentage of body fat, which is determined by underwater weighing and various body measurements in $\mathbf{x}_n$.

\noindent$\bullet$ \textbf{Abalone} dataset contains $417$ samples $(\mathbf{x}_n, y_n)$ with $y_n$ for the age of abalone and $\mathbf{x}_n$ for the physical measurements of abalone, e.g., sex, height, and shell weight.

\subsection{Details for logistic regression}
For logistic regression, consider the binary logistic regression problem
\begin{align}\label{sec6-obj2}
  {\cal L}_m(\bbtheta):=\sum_{n\in{\cal N}_m}\log\left(1+\exp(-y_n\mathbf{x}_n^{\top}\bbtheta)\right)+\frac{\lambda}{2}\|\bbtheta\|^2.
\end{align}
where $\lambda=10^{-3}$ is the regularization constant.

 \noindent\textbf{Real datasets.} Performance is tested on the following datasets; see a summary in Table \ref{tab:2}.

\noindent$\bullet$ \textbf{Ionosphere} dataset \citep{sigillito1989} is to predict whether it is a ``good'' radar return or not -- ``good'' if the features in $\mathbf{x}_n$ show evidence of some structures in the ionosphere.

\noindent$\bullet$ \textbf{Adult} dataset \citep{kohavi1996} contains samples that predict whether a person makes over $50K$ a year based on features in $\mathbf{x}_n$ such as work-class, education, and marital-status.

\noindent$\bullet$ \textbf{Derm} dataset \citep{guvenir1998} for differential diagnosis of erythemato-squaxous diseases, which is determined by clinical and histopathological attributes in $\mathbf{x}_n$ such as erythema, family history, focal hypergranulosis and melanin incontinence.


\end{document}